\documentclass[letterpaper]{article} 
\usepackage{aaai2026}  
\usepackage{times}  
\usepackage{helvet}  
\usepackage{courier}  
\usepackage[hyphens]{url}  
\usepackage{graphicx} 
\usepackage{natbib}  
\usepackage{caption} 
\usepackage{algorithm}
\usepackage{algpseudocode}
\usepackage[utf8]{inputenc} 
\usepackage[T1]{fontenc}    

\usepackage{booktabs}       
\usepackage{amsfonts}       
\usepackage{nicefrac}       
\usepackage{microtype}      
\usepackage{xcolor}         
\def \OURS {PGED}

\usepackage{tikz}
\usepackage{caption}
\usepackage{amsmath}
\usetikzlibrary{arrows.meta, positioning}

\usepackage{tabularx}

\usepackage{enumerate}
\usepackage{thmtools}
\usepackage{thm-restate}

\usepackage[textsize=small]{todonotes}
\usepackage{amsmath}
\usepackage{amssymb}
\usepackage{mathtools}

\newtheorem{theorem}{Theorem}
\newtheorem{lemma}{Lemma}
\newtheorem{proof}{Proof}[section]

\newtheorem{proposition}{Proposition}
\newtheorem{remark}{Remark}
\newtheorem{assumption}{Assumption}

\usepackage{makecell}
\usepackage{multirow}
\usepackage{nicefrac}       
\usepackage{graphicx}
\usepackage{subfigure} 
\usepackage{dsfont}
\usepackage{caption}

\def\ra{\rightarrow}
\def\MAS{\mathrm{MAS}}
\def\wg{\widehat{G}}
\newcommand{\mc}[1]{\mathcal{#1}}
\newcommand{\abs}[1]{\left|#1\right|}
\newcommand{\Sp}[1]{\left(#1\right)}

\usepackage{tcolorbox}

\tcolorboxenvironment{datageneration}{
  colback=rgb{0.7765,0.9373,0.8078}
  boxrule=1pt, 
  boxsep=1pt,
  left=2pt,right=2pt,top=2pt,bottom=2pt,
  oversize=2pt,
  sharp corners,
  before skip=\topsep,
  after skip=\topsep,
}

\tcbset{
  mybox/.style={
    colback=blue!5,
    boxrule=1pt, 
    boxsep=1pt,
    left=10pt, 
    right=10pt, 
    top=10pt, 
    bottom=10pt, 
    oversize=2pt,
    arc=4pt, 
    before skip=10pt, 
    after skip=10pt, 
    baseline=3.5ex 
  }
}

\usepackage{newfloat}
\usepackage{listings}
\DeclareCaptionStyle{ruled}{labelfont=normalfont,labelsep=colon,strut=off} 
\floatstyle{ruled}
\newfloat{listing}{tb}{lst}{}
\floatname{listing}{Listing}
\title{Towards Acyclic Preference Evaluation of Language Models\\via Multiple Evaluators}

\author{
    Zhengyu Hu\textsuperscript{\rm 1},
    Jieyu Zhang\textsuperscript{\rm 2}, 
    Zhihan Xiong\textsuperscript{\rm 2}, 
    Alexander Ratner\textsuperscript{\rm 2}, 
    Kaize Ding\textsuperscript{\rm 3},
    Ranjay Krishna\textsuperscript{\rm 2}
}

\affiliations{
    \textsuperscript{\rm 1} Hong Kong University of Science and Technology, Clear Water Bay, Hong Kong SAR, China\\
    \textsuperscript{\rm 2} University of Washington, Seattle, WA, USA\\
    \textsuperscript{\rm 3} Northwestern University, Evanston, IL, USA
}

\usepackage{bibentry}

\begin{document}

\maketitle
\begin{abstract}

Despite the remarkable success of Large Language Models (LLMs), evaluating their outputs' quality regarding \emph{preference} remains a critical challenge.
While existing works usually leverage a strong LLM as the judge for comparing LLMs' response pairwisely, such a single-evaluator approach is vulnerable to \emph{cyclic preference}
, i.e.,  output A is better than B, B than C, but C is better than A,
causing contradictory evaluation results.
To address this, we introduce \OURS{} 
(Preference Graph Ensemble and Denoising), a novel approach that leverages multiple model-based evaluators to construct preference graphs, and then ensembles and denoises these graphs for acyclic, non-contradictory evaluation results. 
We provide theoretical guarantees for our framework, demonstrating its efficacy in recovering the ground truth preference structure. 
Extensive experiments on ten benchmarks demonstrate \OURS{}'s superiority in three applications: 1) model ranking for evaluation, 2) response selection for test-time scaling, and 3) data selection for model fine-tuning. Notably, \OURS{} combines small LLM evaluators (e.g., Llama3-8B, Mistral-7B, Qwen2-7B) to outperform strong ones (e.g., Qwen2-72B), showcasing its effectiveness in enhancing evaluation reliability and improving model performance.
\end{abstract}

\section{Introduction}
\label{sec:intro}

Large Language Models (LLMs) have rapidly advanced various areas of artificial intelligence, particularly natural language processing and decision-making~\citep{Wu2023AutoGenEN,li2023camel}. As LLMs become increasingly capable, effective evaluation becomes critical~\citep{siska2024examining,boyeau2024autoeval,chatzi2024prediction}. Among evaluation methods, preference evaluation plays a critical role in both evaluating and aligning LLMs~\citep{rafailov2024direct, yuan2024self, dubois2024alpacafarm}. A common practice is to rely on a strong LLM (e.g., GPT-4~\citep{achiam2023gpt}) as the judge to conduct pairwise comparisons~\citep{alpaca_eval,chen2023alpagasus}.

\begin{figure*}[!t]
\centering
\includegraphics[width=0.8\linewidth]{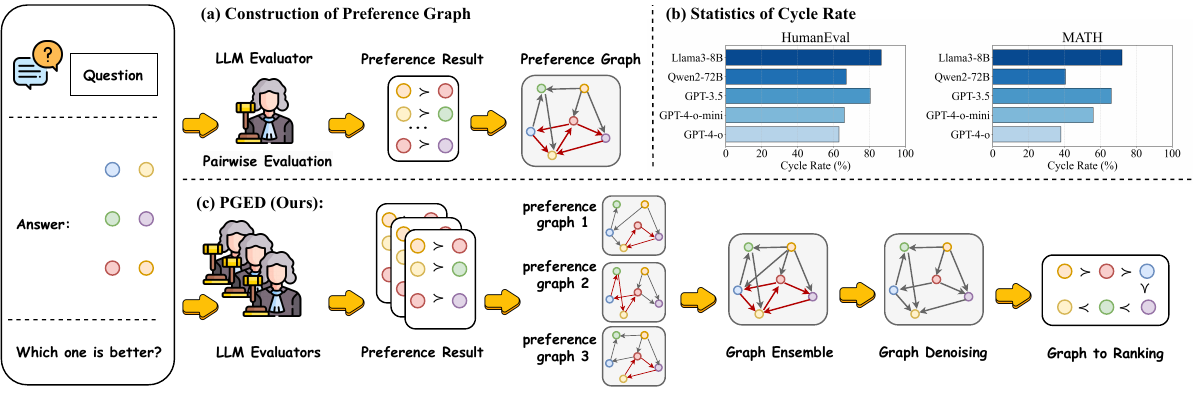}
\caption{
(a) A preference graph exhibiting cyclic inconsistencies (e.g., A $\succ$ B $\succ$ C $\succ$ A), which violate transitivity.
(b) Empirical results showing that even advanced LLMs (e.g., GPT-4-o) exhibit significant noise in preference judgments, leading to inconsistent evaluations.
(c) Overview of our proposed framework, \OURS{}, which ensembles multiple preference evaluators and applies denoising to recover a directed acyclic graph.
}
\label{fig:circle_cnt}
\end{figure*}

However, such model-based evaluator setups often lead to \textit{cyclic preferences}, where inconsistent rankings emerge, for instance, preferring A over B, B over C, yet C over A~\citep{naresh2024curatron,zhang2024contrasolver}. 
These cyclic patterns violate the transitivity assumption of preferences established in prior work~\citep{ouyang2022training, song2024preference, hou2024large, liu2024aligning}, thereby undermining the reliability of evaluation results.
We model this \emph{conflicting preference} using a \textit{preference graph}, where nodes represent responses and directed edges indicate  pairwise preferences. 
Cycles in such graphs (e.g., A $\succ$ B $\succ$ C $\succ$ A) reflect evaluation noise, as shown in Figure~\ref{fig:circle_cnt}(a). Ideally, a  preference graph should be a directed acyclic graph (DAG) to maintain consistency.
Empirically, we evaluated 10 Llama3-70B~\citep{llama3modelcard} responses on HumanEval~\citep{chen2021evaluating} and MATH~\citep{hendrycks2021measuring}, using GPT-4-o, GPT-4-o-mini, GPT-3.5~\citep{achiam2023gpt}, Qwen2-72B~\citep{yang2024qwen2}, and Llama3-8B as judges. 
Even with GPT-4-o, 64\% of preference graphs in HumanEval and 38\% in MATH contained cycles, (Figure~\ref{fig:circle_cnt}(b)), demonstrating the persistent noise in LLM-based preference evaluations and motivating the need for a more robust approach.

To address this, we propose a novel framework, \OURS{} (\textbf{P}reference \textbf{G}raph \textbf{E}nsemble and \textbf{D}enoising). 
Our method involves two key steps: 
(1) ensembling multiple preference evaluators to mitigate noise introduced by individual evaluators and 
(2) applying a denoising process to the resulting preference graph. By aggregating evaluations from multiple individual evaluators, we "average out" the noise and biases, resulting in a more robust approximation of the true preference structure. The denoising step further refines this aggregated graph by removing inconsistencies, ensuring the final preference graph is more reliable for downstream tasks. 
The overall process of \OURS{} is illustrated in Figure~\ref{fig:circle_cnt} (c).
We provide a theoretical analysis demonstrating the soundness of \OURS{}, showing that by treating each individual preference graph as a random perturbation of a ground truth DAG, our ensemble and denoising framework can recover the ground truth DAG with high probability.
To validate the practical efficacy of \OURS{},  we conduct ten tasks ranging from  response selection, model ranking to model alignment tasks, utilizing ten widely recognized benchmark datasets.
In these experiments, \OURS{} consistently outperformed baseline methods. 
Additionally, \OURS{} demonstrated substantial gains in scenarios where combining preference graphs from small evaluators surpassed the performance of even stronger individual evaluators. 
Specifically, when using using Llama3-8B, Mistral-7B, and Qwen2-7B as evaluators, \OURS{} exceeded the performance of using the  Qwen2-72B  in response selection task.
These results highlight \OURS{}'s ability to mitigate preference noise, improve consistency, and enhance model performance across diverse evaluation settings.
Our contributions are summarized as follows:
\begin{itemize}
    \item We propose \OURS{}, a novel framework that ensembles multiple preference evaluators and denoises the resulting preference graphs to produce acyclic and reliable evaluation outputs.
    \item We provide theoretical guarantees that \OURS{} can recover the ground-truth preference DAG under a reasonable noise model, thereby offering provable consistency.
    \item We conduct extensive experiments on ten benchmark datasets and three key tasks, model ranking, response selection, and data selection for fine-tuning, demonstrating that \OURS{} significantly improves evaluation robustness. Notably, \OURS{} using small models outperforms even strong single evaluators like Qwen2-72B.
\end{itemize}

\section{Related Work}

\paragraph{Preference evaluation.}

Classical preference aggregation includes Elo, which iteratively updates item ratings from pairwise outcomes~\citep{Jiang2023LLMBlenderEL}; Bradley–Terry (BT/BTL) models, which assign latent utilities to explain win probabilities~\citep{Bradley1952RankAO,Lu2011LearningMM}; and MergeSort-style procedures, which sort items via pairwise comparisons with \(O(n\log n)\) queries~\citep{Shyamasundar1996IntroductionTA}.
LLM-based preference evaluation increasingly relies on a strong model (e.g., GPT-4) as a zero-shot, reference-free judge of weaker models’ outputs~\citep{Shen2023LooseLS,dubois2024alpacafarm}. 
PRD combines peer ranking over pairwise preferences with peer discussion between LLMs to reach consensus~\citep{Li2023PRDPR}. 
ChatEval forms a multi-agent referee team that debates and evaluates responses for more reliable assessments~\citep{Chan2023ChatEvalTB}.
However, LLM judges can induce contradictory preferences and cycles when aggregating across outputs, yielding noisy and inconsistent preference graphs~\citep{naresh2024curatron,zhang2024contrasolver}.
SPA~\citep{kadekodiselective} produces partial-order rankings that abstain on disputed pairs to balance comparability and disagreement, while ContraSolver~\citep{zhang2024contrasolver} builds a preference graph with self-annotations and flags edges likely responsible for contradictions.
In contrast, our approach ensembles multiple evaluator-induced graphs and denoises them via a weighted feedback-arc-set objective to recover an acyclic structure suitable for robust ranking.

\paragraph{Weak supervision.}
The concept of weak  supervision originates from the need to leverage noisy or partial labels in machine learning tasks, enabling the development of more robust models from imperfect data~\citep{zhang2023leveraging}. 
In LLMs, weak-to-strong supervision aids AI alignment by allowing weaker models to improve strong ones, enhancing performance without extensive data and supporting scalable oversight~\citep{zheng2024weak,guo2024improving,tong2024optimizing}.
Similarly, in task-oriented LLMs, weak-to-strong learning improves LLM's ability by enabling strong models to refine their data autonomously, boosting perfoArmance without extensive high-quality input~\citep{yang2024weak}.
Through weak-to-strong supervision, LLM performance can be significantly improved by iteratively transforming low-quality labels into more reliable ones, leading to more effective model training and robust outputs~\citep{zakershahrak2024explanation,lang2024theoretical}.
Recent work also shows weak LLMs can rival human feedback quality, enabling scalable and cost-efficient alignment~\citep{tao2024your, li2025strong}.

\section[PGED]{Preference Graph Ensemble and Denoising}
\label{sec:basic_operations}

In this section, we first introduce key definitions and assumptions underlying our approach, including the formal definition of a preference graph and the assumption of preference transitivity (Section~\ref{sec:Preliminarily}). 
Building on these preliminaries, we then present our proposed framework, \OURS{}, which consists of three main steps: (1) \emph{Graph Ensemble} (Section~\ref{sec:graphensemble}), aggregating multiple evaluators' preference graphs into a single unified structure; (2) \emph{Graph Denoising} (Section~\ref{sec:graph_denoise}), removing cycles to ensure an acyclic, consistent preference structure; and (3) \emph{Graph-to-Ranking} (Section~\ref{sec:graphtorank}), converting the denoised graph into a reliable candidate ranking.
The overall process of \OURS{} is illustrated in Figure~\ref{fig:circle_cnt}(c).

\subsection{Preliminarily}
\label{sec:Preliminarily}
\paragraph{Transitivity of Preferences.}
In  preference evaluation, we assume that the \emph{ground-truth} preference relation is transitive. 
That is, for any distinct $u,v,w \in V$, if $u \succ v$ and $v \succ w$, then it must hold that $u \succ w$:
\begin{equation}
(u \succ v) \wedge (v \succ w) \Longrightarrow (u \succ w).
\end{equation}
This transitivity assumption ensures that pairwise comparisons can be consistently embedded into a global ranking.

\paragraph{Preference Graph.} A preference graph is a directed graph \( G_P = (V, A, w) \), where \( V = \{v_1, v_2, \ldots, v_n\} \) represents \( n \) candidates, \( A \subseteq V \times V \) is a set of directed arcs indicating pairwise preferences, and \( w: A \rightarrow \mathbb{R}^+ \) assigns weights to arcs, representing preference strength.
For distinct \( u, v \in V \), an arc \( (u, v) \in A \) exists if \( w(u, v) > 0 \).

\subsection{Graph Ensemble}
\label{sec:graphensemble}
Given \(k\) weighted preference graphs \(G_i=(V,A_i,w_i)\) on a common vertex set \(V\), define
\begin{equation}
\begin{aligned}
A_E \quad &\triangleq \Big\{(u,v)\in V\times V:\ \sum_{i=1}^{k} w_i(u,v) > 0\Big\},\\
w_E(u,v) \quad &\triangleq \sum_{i=1}^{k} w_i(u,v),\quad (u,v)\in A_E,
\end{aligned}
\end{equation}
where  $k$ is the total number of preference sources and $w_i(u, v)$ is the preference result from the $i$-th source and \(w_i(u,v)=0\) if \((u,v)\notin A_i\). The ensemble graph is \(G_E=(V,A_E,w_E)\).

\subsection{Graph Denoising}
\label{sec:graph_denoise}
To ensure consistency in the aggregated preference graph, we apply a denoising step that transforms a potentially cyclic weighted digraph \(G=(V,A,w)\) into a directed acyclic graph (DAG). We cast this as the classical \emph{weighted feedback arc set} (WFAS) problem~\citep{gabow1995centroids}, which seeks a minimum-total-weight set of arcs whose removal renders the graph acyclic. Formally,
\begin{align}
\label{eq:wfas}
R^*(G) \;=\; \arg\min_{R \subseteq A} \sum_{(u,v) \in R} w(u,v) \\
\text{s.t.} \quad (V, A \setminus R)\ \text{is acyclic}. \nonumber
\end{align}

A convenient formulation uses vertex sequencing. Given an ordering \(s=\{v_1,\ldots,v_n\}\), the induced feedback arc set \(R(s)\) comprises all arcs that violate the order, i.e., \((v_j \to v_i)\) with \(j>i\). The denoising problem then reduces to finding \(s^*\) that minimizes the total weight of backward edges. Computing the optimal ordering is NP-hard~\citep{karp2010reducibility}, so we adopt an efficient greedy procedure tailored to weighted graphs.
Our algorithm iteratively builds a total order by removing structurally informative vertices. At each step, we first peel all \emph{sinks} (zero out-degree) and place them at the end of the order, then peel all \emph{sources} (zero in-degree) and place them at the beginning. When neither exists, we choose a vertex \(u\) maximizing the weighted degree difference
\begin{equation}
\delta(u) \;=\; d^+(u) - d^-(u),
\end{equation}
where \(d^+(u)=\sum_{(u,v)\in A} w(u,v)\) and \(d^-(u)=\sum_{(v,u)\in A} w(v,u)\). Repeating until all vertices are removed yields a total order; the backward arcs under this order form an approximate feedback arc set \(R(s)\). We then construct and \emph{return only} the denoised graph \(G'=(V, A \setminus R(s), w)\), which is a subgraph of the input and is typically sparse rather than fully connected. The full procedure is summarized in Algorithm~\ref{alg:pgD}.
Its theoretical guarantee can be seen in Appendix~\ref{sec:proof_wfas_convergence}.

\begin{algorithm}[t]
\caption{Preference Graph Denoising for \OURS{}}
\label{alg:pgD}
\begin{algorithmic}[1]
\State \textbf{Input:} Weighted digraph \(G = (V, A, w)\)
\State \textbf{Output:} Denoised graph \(G' = (V, A', w)\)
\State Let \(A_0 \leftarrow A\); initialize \(s_1 \gets [\ ],\ s_2 \gets [\ ]\)
\While{\(V \neq \emptyset\)}
    \While{\(\exists\) sink \(u\) in \(G\)}
        \State Prepend \(u\) to \(s_2\)
        \State Remove \(u\) and its incident arcs from \(G\)
    \EndWhile
    \While{\(\exists\) source \(u\) in \(G\)}
        \State Append \(u\) to \(s_1\)
        \State Remove \(u\) and its incident arcs from \(G\)
    \EndWhile
    \If{\(V = \emptyset\)} \textbf{break} \EndIf
    \State Select \(u = \arg\max_{v \in V} \big(d^+(v) - d^-(v)\big)\)
    \State Append \(u\) to \(s_1\)
    \State Remove \(u\) and its incident arcs from \(G\)
\EndWhile
\State Concatenate \(s = s_1 \Vert s_2\)
\State Compute \(R(s) = \{(v_j \to v_i) \in A_0 : j>i \text{ in } s\}\)
\State Set \(A' = A_0 \setminus R(s)\)
\State \Return \(G' = (V, A', w)\)
\end{algorithmic}
\end{algorithm}

\subsection{Graph to ranking}
\label{sec:graphtorank}

Given a DAG \( G = (V, A, w) \), we derive a ranking by computing the descendant count \( \text{desc}(v) \) for each vertex \( v \), defined as the number of vertices reachable from \( v \):
\begin{equation}
\text{desc}(v) = \left| \{ u \in V : v \rightarrow u \} \right|,
\end{equation}
where \( v \rightarrow u \) denotes a directed path. Vertices are ranked based on \( \text{desc}(v) \), with ties broken lexicographically:
\begin{equation}
v_1 \succ v_2 \succ \cdots \succ v_n.
\end{equation}
This ranking reflects both individual preferences and their relative strengths in the graph.

\section{Downstream Tasks}
\label{subsec: application}
We apply \OURS{} to three tasks: Response Selection (selecting the best response from LLM-generated candidates), Model Ranking (ranking models based on task performance), and Model Alignment (identifying the best instruction-response pairs for training).
Details on how preference graphs are constructed for each of the three settings can be found in Appendix~\ref{sec:preferenceConstruct}.

\subsection{Response Selection}
\label{subsec:response_selection}
For each question \(q \in Q\), a model \(\mathcal{M}\) generates \(n\) candidate answers \(C_q=\{ans_1,\ldots,ans_n\}\). A set of evaluators \(\mathcal{A}=\{a_1,\ldots,a_k\}\) provides pairwise preferences over \(C_q\). For each evaluator \(a\), we construct a weighted preference graph \(G_a^q=(V_q,A_a,w_a)\), where \(V_q=\{v_1,\ldots,v_n\}\) indexes candidates and \(w_a(u,v)\) accumulates the number of wins of \(u\) over \(v\) (ties are ignored). 
We then apply \OURS{}: (i) \emph{Graph Ensemble} aggregates \(\{G_a^q\}_{a\in\mathcal{A}}\) into a single graph \(G_E^q\) (Section~\ref{sec:graphensemble}); (ii) \emph{Graph Denoising} removes a minimum-weight set of feedback arcs to produce a DAG \(G'_q\) (Section~\ref{sec:graph_denoise}); and (iii) \emph{Graph-to-Ranking} converts \(G'_q\) into a total order \(\mathcal{R}_q=\{v_1 \succ \cdots \succ v_n\}\) (Section~\ref{sec:graphtorank}). The top-ranked candidate is selected as \(ans_q^*\). Repeating this for all \(q\in Q\) yields \(ans^*=\{ans_1^*,\ldots,ans_t^*\}\).

\subsection{Model Ranking}
\label{subsec:model_ranking}
Given a set of models \(M=\{\mathcal{M}_1,\ldots,\mathcal{M}_n\}\) and questions \(Q=\{q_1,\ldots,q_t\}\), our goal is to produce a global ranking of \(M\). For each question \(q\in Q\) and evaluator \(a\in\mathcal{A}\), we construct a weighted preference graph \(G_a^q=(V_q,A_a,w_a)\), where \(V_q=\{m_1,\ldots,m_n\}\) indexes model outputs \(\mathcal{M}_i(q)\) and \(w_a(u,v)\) accumulates wins of \(u\) over \(v\) (ties are ignored).
We then apply \OURS{} per question: (i) \emph{Graph Ensemble} aggregates \(\{G_a^q\}_{a\in\mathcal{A}}\) into \(G_E^q\) (Section~\ref{sec:graphensemble}); (ii) \emph{Graph Denoising} yields a DAG \(G'_q\) (Section~\ref{sec:graph_denoise}); and (iii) \emph{Graph-to-Ranking} returns a total order \(\mathcal{R}_q\) over \(V_q\) (Section~\ref{sec:graphtorank}). Finally, we aggregate \(\{\mathcal{R}_q: q\in Q\}\) into an overall model ranking \(\mathcal{R}^*\) using a ranking-ensemble procedure (Appendix~\ref{sec:rank_ensemble_method}).

\subsection{Model Alignment}
\label{subsec:model_alignment}
For each instruction \(x \in X\), let the candidate set be \(Y_x=\{y_1,\ldots,y_n\}\). Evaluators \(\mathcal{A}\) provide pairwise preferences over \(Y_x\). For each \(a\in\mathcal{A}\), we build a weighted preference graph \(G_a^x=(V_x,A_a,w_a)\), where \(V_x=\{v_1,\ldots,v_n\}\) indexes responses and \(w_a(u,v)\) aggregates wins of \(u\) over \(v\) (ties ignored
).
Applying \OURS{}, ensemble, denoising, and graph-to-ranking, produces a total order \(\mathcal{R}_x\); we select the top response \(y_x^*\) for instruction \(x\). Repeating this for all \(x\in X\) yields the aligned training set \(\{(x, y_x^*) : x \in X\}\).

\section{Theoretical Analysis}
\label{sec:theory}
In this section, we provide a theoretical foundation for our method, showing that by modeling preference graphs as random perturbations of a ground truth DAG, \OURS{} can reliably recover the true structure through graph ensemble and denoising with high probability, demonstrating its robustness in handling noisy evaluations.
Theoretically, we treat each of our preference graph as a random perturbation of some ground truth DAG $G=(V, A)$. Specifically, we consider a random graph generator $\mathcal{G}(G, \delta_1, \delta_2)$ with parameters $\delta_1, \delta_2\in [0, 1]$ such that $G_i=(V_i, A_i)\sim \mathcal{G}(G, \delta_1, \delta_2)$ satisfies $V_i=V$. 

Furthermore, for each $u, v\in V$ with $u\neq v$,

\begin{enumerate}[1)]
    \item If $(u\rightarrow v)\in A$, then 
    \begin{equation*}
    \begin{aligned}
        \mathbb{P}((u\rightarrow v)\in A_i) &= 1-\delta_1, \\
        \mathbb{P}((v\rightarrow u)\in A_i) &= \delta_1.
    \end{aligned}
    \end{equation*}

    \item If $(u\rightarrow v), (v\rightarrow u)\notin A$, then
    \begin{equation*}
    \begin{aligned}
        \mathbb{P}((u\rightarrow v), (v\rightarrow u)\notin A_i) &= 1-\delta_2, \\
        \mathbb{P}((u\rightarrow v)\in A_i) &= \frac{\delta_2}{2}, \\
        \mathbb{P}((v\rightarrow u)\in A_i) &= \frac{\delta_2}{2}.
    \end{aligned}
    \end{equation*}
\end{enumerate}

That is, each edge in $E$ has probability $\delta_1$ of being flipped and each pair of unconnected nodes has probability $\delta_2$ of being connected with a random direction. 

Now, given that $G_1, \dots, G_N\overset{\text{i.i.d.}}{\sim}\mathcal{G}(G, \delta_1, \delta_2)$, we will show that to some extent our combination of graph ensemble and graph denoising can indeed provably recover the ground truth DAG $G$. For simplicity, all edges in $G_1, \dots, G_N$ and $G$ are considered equal weighted. Meanwhile, we use $\text{MAS}(\cdot)$ to denote the graph obtained by denoising, which stands for the maximum acyclic subgraph (MAS). Then, we have the following theorem.

\begin{restatable}{theorem}{ensembleFirst}
    \label{theo:ensemble_first}
    Suppose $G_1, \dots, G_N\overset{\text{i.i.d.}}{\sim}\mathcal{G}(G, \delta_1, \delta_2)$ for some ground truth $G=(V, A)$. Let $\widehat{G}$ be the graph ensembled from $G_1, \dots, G_N$ by operations defined in Section \ref{sec:basic_operations}. Then, as long as $\delta_1=0.5-\epsilon$ for some $\epsilon>0$, we have
    \begin{equation*}
    \begin{split}
    \mathbb{P}\left( G\subseteq\text{MAS}(\widehat{G}) \right) 
    &\geq 1 - 2|A|\exp\Sp{-\frac{N\epsilon^2}{2}} \\
    &\quad - 2U\exp\Sp{-\frac{N\epsilon^2}{6U^2\delta_2 + 2U\epsilon}},
    \end{split}
    \end{equation*}
    where $G\subseteq\text{MAS}(\widehat{G})$ represents that $G$ is a subgraph of $\text{MAS}(\widehat{G})$ and $U=\frac{|V|(|V|-1)}{2} - |A|$ is the number of pairs of unconnected nodes in $G$.
\end{restatable}

The full proof is given in Appendix \ref{sec:proof}. From the theorem, we can see that the probability of failure decreases exponentially as the number of samples $N$ increases. Meanwhile, this guarantee only requires $\delta_1<0.5$ and does not place restrictions on $\delta_2$, which are very mild conditions.

\begin{table*}[h]
\captionsetup{skip=0pt} 
\centering
\scalebox{0.74}{

\begin{tabular}{@{}llcccccc@{}}
\toprule

\multicolumn{2}{l}{\textbf{Method}}          & \textbf{HumanEval} & \textbf{AlpacaEval} & \textbf{MATH}  & \textbf{GSM8k} & \textbf{GAIA}  & \textbf{Avg}   \\ \midrule
                                              & Llama3-8B            & 43.90              & 27.29               & 22.08          & 56.67          & 6.78           & 31.34          \\ \cmidrule(l){2-8} 
                                              & Mistral-7B           & 23.17              & 11.80               & 23.25          & 39.83          & 7.03           & 21.01          \\ \cmidrule(l){2-8} 
                                              & Qwen2-7B             & 48.58              & 25.71               & 59.92          & 76.75          & 7.70           & 43.73          \\ \cmidrule(l){2-8} 
                                              & Qwen2-72B            & 57.93              & 29.58               & 72.75          & 84.67          & 11.52          & 51.29          \\ \cmidrule(l){2-8} 
                                              & ContraSolver (Qwen2-72B)         & 65.42              & 31.12               & 74.95          & 86.84          & 12.22          & 54.11          \\ \cmidrule(l){2-8} 
                                              & ListPreference       & 61.52              & 31.67               & 71.75          & 85.0           & 10.90          & 52.16          \\ \cmidrule(l){2-8} 
\multirow{-7}{*}{\textbf{Single model}}       & Self-consistency     & 60.98              & 29.33               & 73.58          & 84.91          & 8.86           & 51.53          \\ \midrule

                                             & Elo            & 62.21 & 31.34 & 75.01 & 87.21 & 12.36 & 53.63 \\ \cmidrule(l){2-8} 
                                             & Bradley-Terry  & 62.83 & 32.91 & 74.97 & 86.64 & 12.07 &  53.88 \\ \cmidrule(l){2- 8} 
                                             & Merge Sort     & 62.81 & 31.88 & 74.83 & 86.90 & 12.13 & 53.71  \\  \cmidrule(l){2-8}
                                              & Llama3-8B            & 62.19              & 29.31               & 74.27          & 83.16          & 11.31          & 52.04          \\ \cmidrule(l){3-8} 
                                              & with graph denoising & 64.02              & 30.18               & 74.73          & 86.00          & 11.72          & 53.33          \\ \cmidrule(l){2-8} 
                                              & Mistral-7B           & 67.24              & 27.70               & 74.41          & 83.83          & 10.50          & 52.73          \\ \cmidrule(l){3-8} 
                                              & with graph denoising & 68.73              & 29.93               & 74.77          & 83.91          & 10.74          & 53.61          \\ \cmidrule(l){2-8} 
                                              & Qwen2-7B             & 61.58              & 28.69               & 74.50          & 85.41          & 11.11          & 52.25          \\ \cmidrule(l){3-8} 
                                              & with graph denoising & 65.85              & 29.44               & 74.79          & 86.38          & 11.25          & 53.54          \\ \cmidrule(l){2-8} 
                                              & Qwen2-72B            & 60.97              & 31.04               & 74.73          & 86.47          & 12.14          & 53.07          \\ \cmidrule(l){3-8} 
\multirow{-8}{*}{\textbf{Single evaluator}}   
& with graph denoising & 68.90              & 31.17               & 75.33          & 87.45          & 12.26          & 55.02          \\ \midrule
& {Majority Voting } & {66.18}              & {29.57}               & {74.77}          & {86.42}          & {11.72}          & {53.73}          \\ \cmidrule(l){3-8} 
& {Weight Majority Voting } & {66.43} & {30.12} & {74.79} & {86.68} & {11.84} & {53.97} \\ \cmidrule(l){3-8} 
& {PRD}       & {66.53}              & {30.33}               & {74.89}          & {86.75}          & {11.83}          & {54.07}          \\  \cmidrule(l){3-8} 
& {ChatEval}  & {66.51}              & {30.26}               & {74.79}          & {86.85}          & {11.92}          & {54.07}          \\  \cmidrule(l){3-8} 
& {SPA}  & {66.32}              & {30.05}               & {74.82}          & {86.57}          & {11.81}          & {53.91}          \\  \cmidrule(l){3-8} 
& \OURS (w/o denoising)  & 69.25              & 30.98               & 74.29          & 87.17          & 12.68          & 54.87          \\ \cmidrule(l){3-8} 
\multirow{-7}{*}{\textbf{Multiple evaluator}} & \OURS                  & \textbf{71.86} & \textbf{32.95} & \textbf{76.57} & \textbf{89.76} & \textbf{13.74} & \textbf{56.98} \\ \bottomrule
\end{tabular}

}

\caption{
Performance comparison of response selection methods across five benchmarks. \OURS{} consistently outperforms baseline methods, highlighting the effectiveness of graph denoising and multi-evaluator aggregation. 
\textit{Single model} denotes directly using the evaluator to answer the question. \textit{Single evaluator} and \textit{Multi evaluator} refer to selecting the best response from ten Qwen2-72B-generated candidates using a single or multiple evaluators, respectively.
}
\label{table:ansselectionQwen}
\end{table*}

\section{Response Selection for Test-time Scaling}
\label{sec:ans-select}

\paragraph{Experiment Setup.}
In this section, we evaluate the performance of \OURS{} on five benchmarks: HumanEval~\citep{chen2021evaluating}, AlpacaEval~\citep{alpaca_eval}, MATH~\citep{hendrycks2021measuring}, GSM8k~\citep{chen2021evaluating}, and {GAIA}~\citep{mialon2023gaia}. 
The Qwen2-72B~\citep{yang2024qwen2} model ($\mathcal{M}$) generates ten candidate responses per question, and we assess the effectiveness of different methods in selecting the best response.
We evaluate performance using three setups. 
First, in the \textit{single model} setting, the baselines include ContraSolver\citep{zhang2024contrasolver}, Self-consistency\citep{wang2022self}, and direct evaluation with  models (Llama3-8B, Mistral-7B, Qwen2-7B and Qwen2-72B).
Additionally, we include a baseline called ListPreference, where instead of pairwise comparisons, all candidate responses are input into Qwen2-72B for selecting the most appropriate response.
Then, in the \textit{single evaluator} setting, individual evaluators (Llama3-8B, Mistral-7B, Qwen2-7B, Qwen2-72B) select the best response from \(\mathcal{M}\)’s outputs, with and without applying \OURS{}’s graph denoising. 
We additionally report results for three classical baselines in this setting: {Elo}~\cite{Jiang2023LLMBlenderEL}, {Bradley–Terry} (BTL)~\cite{Bradley1952RankAO,Lu2011LearningMM}, and {MergeSort}~\cite{Shyamasundar1996IntroductionTA}. 
Finally, in the \textit{multiple evaluators} setup, we combine three small evaluators (Llama3-8B, Qwen2-7B, Mistral-7B) to select responses from Qwen2-72B with \OURS{}. 
We additionally include four aggregation baselines: {Majority Voting}, {Weighted Majority Voting}, {PRD}~\cite{Li2023PRDPR}, {ChatEval}~\cite{Chan2023ChatEvalTB}, and SPA~\cite{kadekodiselective}.
We present the results of \OURS{} and its variant (w/o denoising), which ensembles the preference graphs without the denoising step.

Note that, among the baselines we report: (i) in the single model setting, ContraSolver, ListPreference, and Self-consistency use Qwen2-72B as the base model; (ii) in the single evaluator setting, Elo, Bradley–Terry, and MergeSort use Qwen2-72B as the base model; and (iii) in the multiple evaluators setting, Majority Voting, Weighted Majority Voting, PRD, ChatEval and SPA use Llama3-8B, Mistral-7B, and Qwen2-7B as the base model set.
For details on the datasets and baselines, please refer to Appendix~\ref{sec:implementation_details}.

\begin{figure}[h]
\centering
\includegraphics[width=1.0\linewidth]{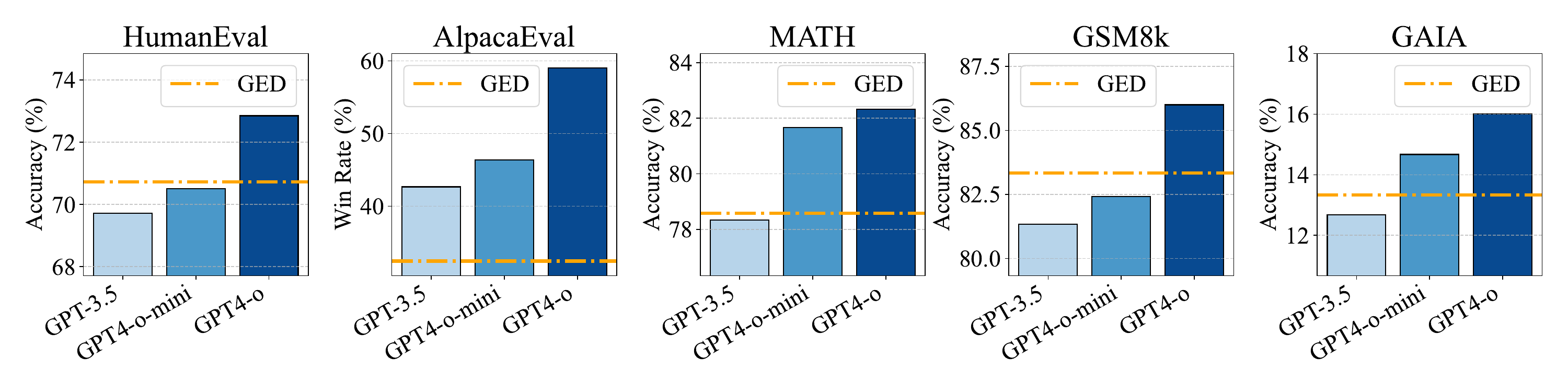}
\captionsetup{skip=0pt} 
\caption{
Comparison of \OURS{} with GPT-3.5, GPT-4-o-mini, and GPT-4-o on 100 randomly selected tasks. 
\OURS{} consistently outperforms GPT-3.5 across all tasks and surpasses GPT-4-o-mini on challenging tasks like HumanEval and GSM8k, showcasing the effectiveness of weak evaluator aggregation with graph denoising.
}
\label{fig:real_API}
\end{figure}

\paragraph{Main results.}
Table~\ref{table:ansselectionQwen} presents the results of the response selection task across five benchmarks. 
\OURS{} consistently outperforms all baseline methods, including  single model evaluations (\textit{single model}), direct response selection by individual models (\textit{single evaluator}) and Multiple evaluator. 
This demonstrates the strength of aggregating preference evaluators with \OURS{}, particularly when coupled with graph denoising, which enhances response quality by filtering out noise and biases.
Furthermore, by combining preference graphs derived from smaller models (Llama3-8B, Mistral-7B, Qwen2-7B), \OURS{} outperforms a much larger evaluator (Qwen2-72B). This underscores the value of ensemble methods in mitigating the limitations of individual evaluators.
Then, the denoising process proves to be crucial for improving consistency and overall response quality. The substantial performance gains observed when using \OURS{} with denoising, compared to both the single evaluator setup and the ensemble without denoising, highlight its importance in refining response selection.
For Majority Voting, while it improves upon individual evaluators, it still underperforms \OURS{}, highlighting \OURS{}'s ability to capture nuanced evaluation signals and reduce inconsistencies.
Additionally, we observed that the ListPreference baseline performed worse than Qwen2-72B as single evaluator, likely due to LLM limitations in handling long-text.
Lastly, to further evaluate \OURS{}, we  compared its performance with GPT-3.5, GPT-4-o-mini, and GPT-4-o. 
Due to computational and API cost constraints, we limited the evaluation to 100 data points for each task.
As shown in Figure~\ref{fig:real_API}, \OURS{} consistently outperformed GPT-3.5 across all tasks and surpassed GPT-4-o-mini on challenging benchmarks like HumanEval and GSM8k. 
These results highlight the superiority of \OURS{}, particularly in leveraging multi-preference evaluators and graph denoising to outperform individual state-of-the-art models.
More discussion on cost is provided in Appendix~\ref{sec:cost_detail}.

\begin{table*}[t]

\centering
\scalebox{0.7}{

\begin{tabular}{@{}llcccccc@{}}
\toprule
\textbf{Model}                                &                      & \textbf{Weight Score} & \textbf{Kemeny} & \textbf{\begin{tabular}[c]{@{}c@{}}Weighted\\ Kemeny\end{tabular}} & \textbf{\begin{tabular}[c]{@{}c@{}}Pairwise\\ Majority\end{tabular}} & \textbf{\begin{tabular}[c]{@{}c@{}}Weighted\\ Pairwise\\ Majority\end{tabular}} & \multicolumn{1}{l}{\textbf{Avg.}} \\ \midrule
                                              & Llama3-70B           & 50.88                 & 60.80           & 60.80                                                              & 62.23                                                                & 61.85                                                                           & 59.31                                                     \\
                                              & with graph denoising & 52.44                 & 62.54           & 62.54                                                              & 63.92                                                                & 62.18                                                                           & 60.72                                                     \\ \cmidrule(l){2-8} 
                                              & Qwen2-72B            & 65.34                 & 59.87           & 67.39                                                              & 66.05                                                                & 66.59                                                                           & 65.04                                                     \\
                                              & with graph denoising & 66.05                 & 70.43           & 70.43                                                              & 72.32                                                                & 72.41                                                                           & 70.32                                                     \\ \cmidrule(l){2-8} 
                                              & Qwen1.5-72B          & 63.64                 & 60.72           & 60.72                                                              & 62.65                                                                & 63.28                                                                           & 62.20                                                     \\
                                              & with graph denoising & 64.81                 & 61.77           & 61.77                                                              & 64.36                                                                & 64.76                                                                           & 63.49                                                     \\ \cmidrule(l){2-8} 
                                              & Mistral-8×7B         & 64.90                 & 68.74           & 68.74                                                              & 73.06                                                                & 72.87                                                                           & 69.66                                                     \\
\multirow{-8}{*}{\textbf{Single evaluator}}   & with graph denoising & 65.47                 & 70.06           & 69.92                                                              & 73.39                                                                & 73.21                                                                           & 70.41                                                     \\ \midrule
                                              & \OURS (w/o ensemble)   & 62.82                 & 68.44           & 68.44                                                              & 69.34                                                                & 67.34                                                                           & 67.27                                                     \\ \cmidrule(l){3-8} 
                                              & \OURS (w/o denoising)  & 64.84                 & 69.23           & 69.81                                                              & 75.35                                                                & 74.37                                                                           & 70.72                                                     \\ \cmidrule(l){3-8} 
\multirow{-3}{*}{\textbf{Multiple evaluator}} & \OURS                  & \textbf{66.59}        & \textbf{71.14}  & \textbf{71.14}                                                     & \textbf{77.17}                                                       & \textbf{76.46}                                                                  & \textbf{72.50}                                            \\ \bottomrule
\end{tabular}

}
\caption{Results of the model ranking task, evaluated using Ranking Correction. Higher correlation values indicate a stronger alignment with the ground truth rankings.}

\label{table:modelrank}
\end{table*}

\section{Model Ranking for Evaluation}
\label{sec:prefer-eval}
\paragraph{Experiment Setup.}
In this section, we evaluate the effectiveness of \OURS{} in the model ranking task within a human preference setting, using the AlpacaEval 2.0~\citep{alpaca_eval}.
We employ 30 widely used models from the AlpacaEval dataset as our model set \( \mathcal{M} \), while the benchmark’s questions form the question set \( Q \). 
The rankings provided by the AlpacaEval benchmark serve as ground truth for evaluating the accuracy of various ranking methods.  
This is justified by AlpacaEval's strong correlation with Chatbot Arena rankings, making it a reasonable proxy for human judgments~\citep{dubois2024length}.
We adopt Ranking Correction, measured by the Spearman rank correlation coefficient, to evaluate the similarity.
To generate rankings, we utilize outputs from the open-source models Llama3-70B, Qwen2-72B, Mistral-8$\times$7B, and Qwen1.5-72B as our evaluators. 
We investigate two variants of \OURS{}: 
(w/o ensemble)  denoises the preference graphs from different evaluators for the same question, converts each into a ranking, and then ensembles these rankings to produce the final output, while  (w/o denoising)  directly ensembles the preference graphs to obtain the final ranking without denoising.
For details on the datasets and baselines, please refer to Appendix~\ref{sec:implementation_details}.

\paragraph{Main results.}
The results, presented in Table~\ref{table:modelrank}, show that \OURS{} outperforms all single-model baselines, highlighting the significant improvement in ranking accuracy achieved by leveraging preference information from multiple evaluators. 
Moreover, \OURS{} surpasses the (w/o ensemble) variant, indicating that generating rankings through graph ensemble first prevents information loss compared to converting individual graphs into rankings. 
When the ensemble graph is not denoised (w/o denoising), residual noise can adversely affect the final ranking quality. 
Additionally, our denoising method also enhances results in single-model settings.

\begin{figure}[h]
\centering
\includegraphics[width=1.0\linewidth]{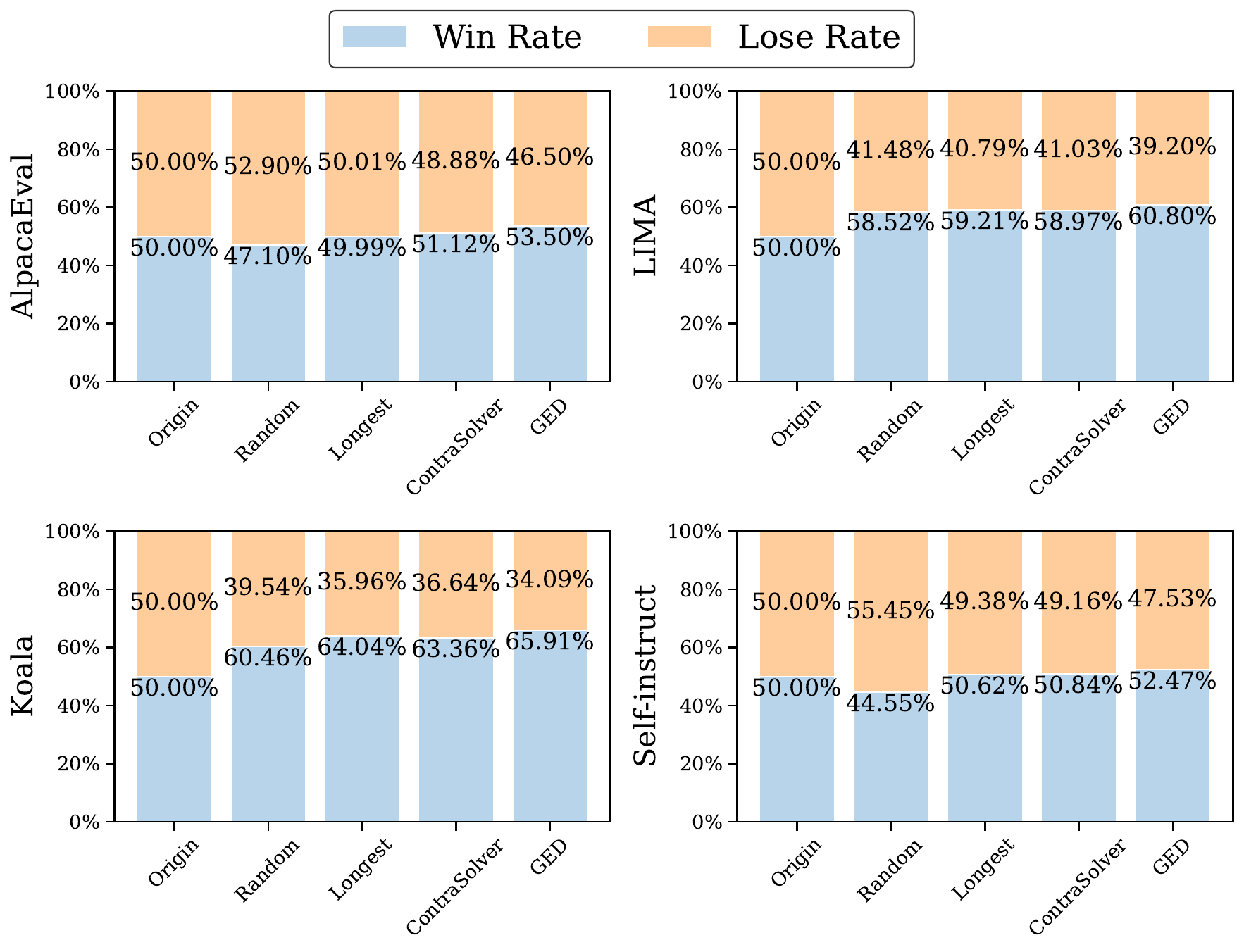}
\captionsetup{skip=0pt} 
\caption{
Performance comparison of different methods (Random, Longest, ContraSolver, and \OURS{}) across multiple benchmarks.
The results show \OURS{} effectively filters low-quality responses, improving performance and model alignment over baselines.
}
\label{fig:tuning}
\end{figure}

\section{Data Selection for Model Fine-tuning}
\label{sec:ans-instruct}

\paragraph{Experiment Setup.}

\begin{table}[t]
\centering
\captionsetup{skip=0pt}
\scalebox{0.7}{
\begin{tabular}{@{}llccccc@{}}
\toprule
\textbf{\makecell{Base \\ Model}} & \textbf{} & \textbf{\makecell{Harmless \\ (base)}} & \textbf{\makecell{Helpful \\ (base)}} & \textbf{\makecell{Helpful \\ (online)}} & \textbf{\makecell{Helpful \\ (rej.)}} & \textbf{Avg.} \\
\midrule
& Origin       & 69.68 & 61.65 & 64.84 & 63.37 & 64.89 \\
& Random       & 69.85 & 63.53 & 67.00 & 65.85 & 66.56 \\
& Longest      & 69.32 & 64.15 & 66.56 & 66.30 & 66.58 \\
& ContraSolver & 69.75 & 63.28 & 66.71 & 65.86 & 66.40 \\
\multirow{-5}{*}{\rotatebox{90}{\textbf{Llama-2-7B}}} & \textbf{Ours} & \textbf{72.27} & \textbf{65.56} & \textbf{69.44} & \textbf{67.60} & \textbf{68.72} \\
\midrule
& Origin       & 61.30 & 59.95 & 65.21 & 63.75 & 62.55 \\
& Random       & 59.85 & 59.34 & 63.38 & 61.87 & 61.11 \\
& Longest      & 61.66 & 59.73 & 64.92 & 62.92 & 62.31 \\
& ContraSolver & 60.87 & 59.83 & 64.23 & 63.26 & 62.05 \\
\multirow{-5}{*}{\rotatebox{90}{\textbf{Mistral-7B}}}  & \textbf{Ours} & \textbf{64.57} & \textbf{64.48} & \textbf{69.66} & \textbf{67.73} & \textbf{66.61} \\
\bottomrule
\end{tabular}
}
\caption{Performance comparison of different methods  on model alignment task across the HH-RLHF benchmark.}
\label{table:tuning}
\end{table}

In this section, we explore the effects of various data selection methods for model alignment on Llama-2-7B~\citep{touvron2023llama} and Mistral-7B~\citep{jiang2023mistral} through instruct tuning. 
Specifically, we randomly sampled 5,000 data points from UltraFeedback~\citep{cui2023ultrafeedback} and used Qwen1.5-14B~\citep{yang2024qwen2} to generate eight responses per data point as instruct data.
We then applied four different methods—Random, Longest~\citep{zhao2024long}, ContraSolver (using Qwen2-72B as the evaluator)~\citep{zhang2024contrasolver}, and our proposed \OURS{}, which leverages Llama3-8B, Mistral-7B, and Qwen2-7B as evaluators—to select a subset of these responses for model alignment training.
The Origin refers to the performance of the base model without alignment.
The models were evaluated on the HH-RLHF~\citep{bai2022training} benchmark, which comprises four subsets: Harmless (base), Helpful (base), Helpful (online), and Helpful (rejection).
For evaluation, we employed the same Reward model as in prior work~\citep{song2024preference,yu2023constructive} to quantify human preference levels.
The results are presented in Table~\ref{table:tuning}.
To ensure a comprehensive assessment, we further evaluated the models—using the Llama-2-7B backbone—on additional benchmarks, including
AlpacaEval 2.0~\citep{alpaca_eval},
LIMA~\citep{zhou2023lima}, 
Koala~\citep{vu2023koala}, and
Self-Instruct~\citep{wang2022self}, in accordance with recent studies~\citep{chen2023alpagasus,zhang2024recost,hu2024rethinking}.
The corresponding results are summarized in Figure~\ref{fig:tuning}.
For details on the datasets and baselines, please refer to Appendix~\ref{sec:implementation_details}.

\paragraph{Main results.}

From Table~\ref{table:tuning}, we observe that \OURS{} consistently outperforms all baseline methods, demonstrating its effectiveness in selecting high-quality responses when multiple answers are available for a given instruction.
When faced with multiple responses \( y_1, y_2, \ldots, y_n \) for a given instruction \( x \), the Random selection method can have a detrimental impact, especially when the quality of the responses is inconsistent. 
This effect is most evident with the Mistral-7B, where Random selection actually performs worse than the Origin, indicating that randomly chosen data points can introduce noise and degrade the model's performance.
Moreover, we find that simply selecting the longest response does not always lead to the best outcomes. 
While longer responses may provide more detailed answers, they are not necessarily better in terms of quality, particularly when both high-quality and low-quality answers exist for the same question. 
This is reflected in the results where the Longest method underperforms compared to both ContraSolver and \OURS{}, emphasizing that response length alone is not always a reliable criterion.
From Figure~\ref{fig:tuning}, we observe similar trends as in Table~\ref{table:tuning}.
\OURS{} consistently outperforms all baselines across various datasets, demonstrating its effectiveness in selecting high-quality responses.
Notably, in AlpacaEval and Self-Instruct, the Random baseline performs worse than the Origin model, highlighting that when response quality varies significantly, poor selection can degrade model performance.
In contrast, \OURS{} leverages preference graphs and denoising techniques to filter out low-quality responses, ensuring more robust and reliable performance, particularly in settings with inconsistent responses, as it removes evaluation noise and leads to more robust performance.

\section{Case Study}

\begin{figure}[h]
    \centering
    \resizebox{0.50\textwidth}{!}{
        \begin{tabular}{cc}
            \subfigure[Raw Preference Graph 1]{\includegraphics[width=0.24\textwidth]{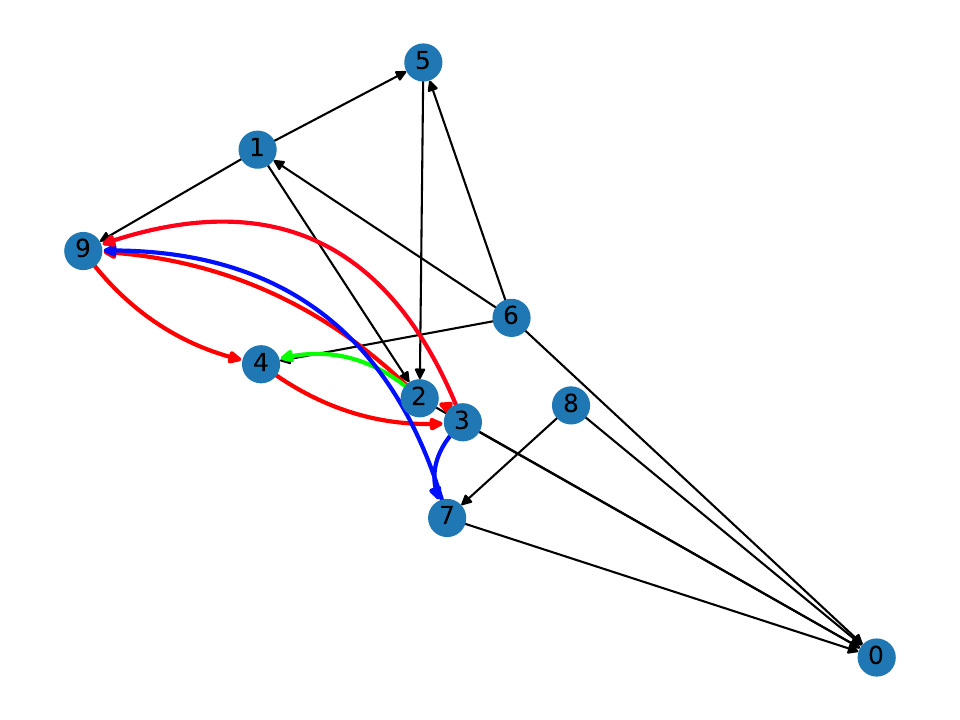}} &
            \subfigure[Denoised Preference Graph 1]{\includegraphics[width=0.24\textwidth]{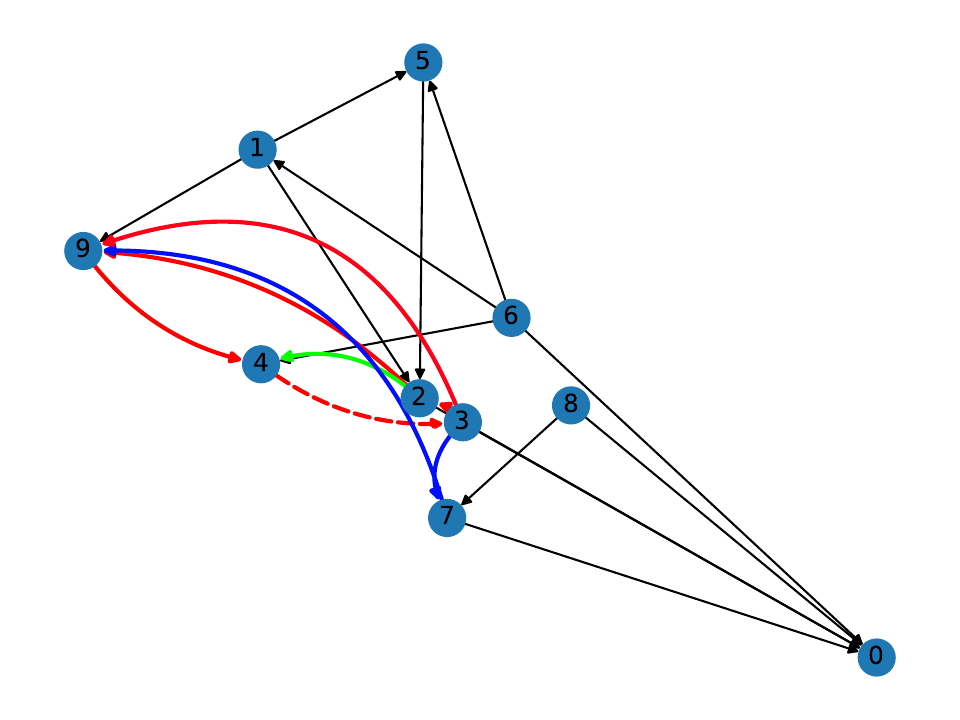}} \\
            \subfigure[Raw Preference Graph 2]{\includegraphics[width=0.24\textwidth]{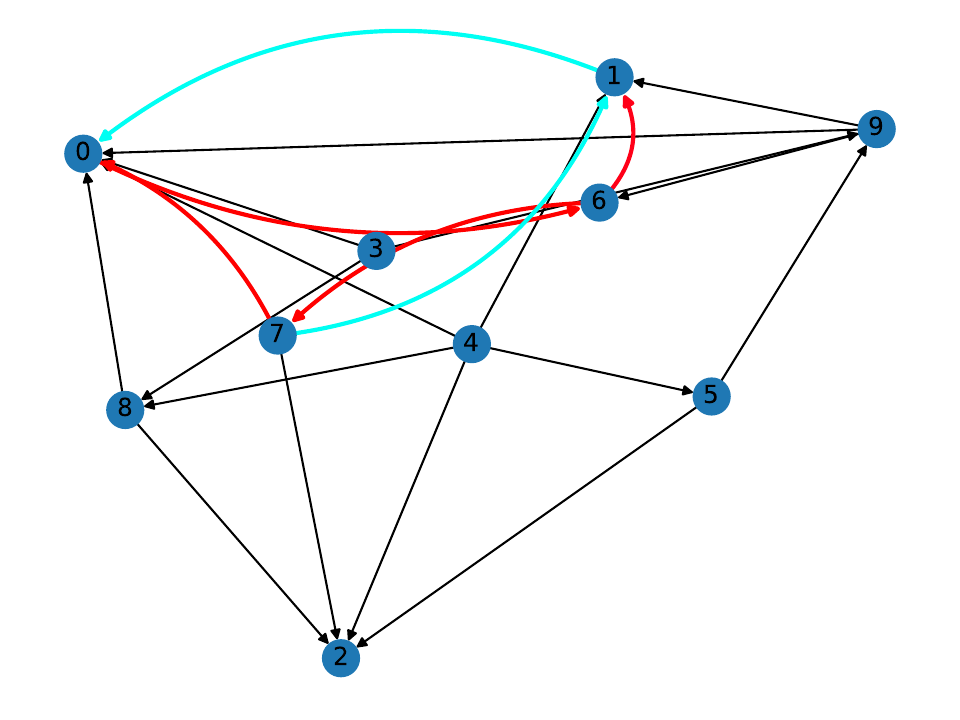}} &
            \subfigure[Denoised Preference Graph 2]{\includegraphics[width=0.24\textwidth]{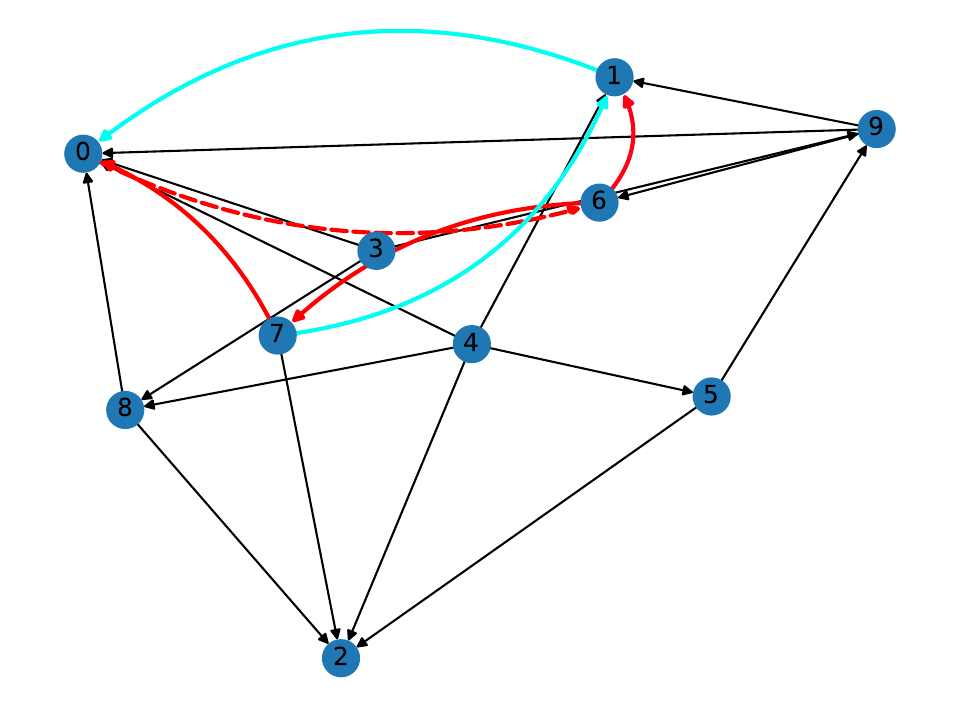}}
        \end{tabular}
    }
    \caption{Case studies showcasing the raw and denoised preference graphs. 
    }
    \label{fig:casestudy}
\end{figure}

In this section, we present two HumanEval case studies illustrating \OURS{} on evaluator-aggregated graphs. Raw graphs are built from Llama3-8B, Mistral-7B, and Qwen2-7B judging ten Qwen2-72B responses (nodes 0–9). Figure~\ref{fig:casestudy} contrasts raw graphs (a, c) with denoised DAGs (b, d); dashed edges indicate removals.
In Case 1, the raw graph contains cycles (e.g., \(9 \succ 4 \succ 3 \succ 9\), \(9 \succ 4 \succ 3 \succ 2 \succ 9\)). Removing a single edge \(4 \to 3\) breaks all cycles, yielding the DAG in Figure~\ref{fig:casestudy}(b).
In Case 2, cycles such as \(7 \succ 0 \succ 6 \succ 7\) and \(7 \succ 1 \succ 0 \succ 6 \succ 7\) are eliminated by deleting \(0 \to 6\), producing the DAG in Figure~\ref{fig:casestudy}(d).
In both cases, alternative single-edge deletions fail to restore acyclicity without further removals. \OURS{} thus identifies a minimal feedback arc set, preserving information while ensuring consistency.

\section{Conclusion}
In this paper, we presented \OURS{}, a framework designed to address inconsistencies in pairwise preference evaluations by LLMs. By employing graph ensemble techniques and denoising, \OURS{} reduces cyclic patterns and enhances the reliability of evaluation outcomes. Our theoretical analysis shows that \OURS{} can recover the ground truth DAG under reasonable conditions, improving consistency in preference rankings.
Extensive experiments across response selection, model ranking, and instruct tuning demonstrate the efficacy of our method. 
\OURS{} consistently outperformed baseline methods in both single-evaluator and multi-evaluator settings, particularly in scenarios where combining small evaluators led to superior results over larger individual evaluators. 
Future work will explore extending \OURS{} to broader evaluation frameworks and applying its principles to more complex decision-making tasks, including multi-agent systems and human-AI interaction.

\bibliography{aaai2026}
\clearpage

\clearpage

\appendix

\section{Implementation Details}
\label{sec:implementation_details}

\subsection{Experimental Setup}
\label{sec:Ex_setup}
All experimental procedures were conducted on a machine equipped with an AMD EPYC 7543 32-Core Processor, 512GB memory, 128 CPUs, and four 80GB NVIDIA A800 GPUs. 
The references to Llama-2-7B, Llama3-70B, Llama3-8B, Mistral-7B, Mistral-8×7B, Qwen2-7B, and Qwen2-72B in the main text refer to the specific models: Llama-2-7b-chat-hf, Meta-Llama-3-70B-Instruct, Meta-Llama-3-8B-Instruct, {Mixtral-7B-Instruct-v0.3, Mixtral-8x7B-Instruct-v0.1}, Qwen2-7B-Instruct, and Qwen2-72B-Instruct. 
We utilized the reward model oasst-rm-2-pythia-6.9b-epoch-1 following prior works~\citep{song2024preference,yu2023constructive}.
Each experiment was repeated three times, and the average performance was reported as the final result.
{Our training script was adapted from the example provided in LlamaFactory
~\citep{zheng2024llamafactory}. }
The training was configured with a batch size of 1 per device, gradient accumulation steps of 4, a learning rate of 1e-5, and the model was trained for 3 epochs, with warmup over 20 steps and a cosine learning rate scheduler.
{For generating diverse responses from LLMs, we followed the configuration in \cite{yuan2024self}, setting  \(T = 0.7\) and \(p = 0.9\).} 
For tasks such as AlpacaEval~\citep{dubois2024alpacafarm}, we used GPT-4-o unless stated otherwise.

\subsection{Details of Evaluator Selection Across Different Tasks}
\label{subsec:evaluator_selection}
In this subsection, we provide more detailed information about the selection of evaluators across different tasks.

\paragraph{Response Selection.}
In the response selection task, we evaluated both single models and ensembles of evaluators using the \OURS method. For single model evaluation, we assessed the standalone performance of Llama3-8B, Mistral-7B, Qwen2-7B, and Qwen2-72B on benchmarks such as HumanEval, AlpacaEval, MATH, GSM8k, and GAIA. This served two purposes: first, to establish the baseline performance of smaller models (Llama3-8B, Mistral-7B, Qwen2-7B) that are later combined using \OURS, and second, to compare against Qwen2-72B, where we tested a baseline approach by randomly selecting one response from the 10 it generated for each question. For single evaluator evaluation, each model (Llama3-8B, Mistral-7B, Qwen2-7B, and Qwen2-72B) was also used as an evaluator to rank responses generated by Qwen2-72B. Notably, Qwen2-72B acted as a self-evaluator, selecting the best response from its own generated outputs. Finally, for multiple evaluator evaluation, our \OURS method combined the evaluations of Llama3-8B, Mistral-7B, and Qwen2-7B to rank responses.

\paragraph{Model Ranking.}
For the model ranking task, we selected larger models as evaluators to ensure alignment with rankings produced by GPT-4. Specifically, we used Llama3-70B, Qwen2-72B, Qwen1.5-72B, and Mistral-8×7B as single evaluators. This choice was guided by two factors: performance considerations and practical feasibility. Larger models generally produce more reliable rankings, closely aligning with GPT-4, and the AlpacaEval benchmark, containing 805 tasks, makes the computational cost of using larger models acceptable. In the multiple evaluator setting, our \OURS method aggregated the evaluations from these four larger models to produce robust and consistent rankings. The combination of these high-capacity models ensures that our approach yields rankings that are both accurate and consistent across tasks.

\paragraph{Instruction Tuning.}
In the instruction tuning task, the objective was to perform data selection for model training. For this task, we employed Llama3-8B, Mistral-7B, and Qwen2-7B as evaluators. These models were selected because they balance computational efficiency and evaluation performance, making them suitable for iterative instruction tuning processes. The evaluators’ pairwise preferences were aggregated using \OURS to identify the most appropriate responses for training. This approach ensured that the selected data pairs reflected a consensus among the evaluators while keeping computational costs manageable.

\subsection{Evaluation Settings: Single Model vs. Single Evaluator}
\label{subsec:evaluation_settings}
The "single model" setting evaluates each model's (Llama3-8B, Mistral-7B, Qwen2-7B, Qwen2-72B) outputs directly on benchmarks like HumanEval, AlpacaEval, MATH, GSM8k, and GAIA, without selection or modification. In contrast, the "single evaluator" setting uses these models to select the best response from ten candidates generated by Qwen2-72B, assessing their evaluation capability. The key difference is that the single model setting focuses on generation quality, while the single evaluator setting assesses evaluation ability.

\subsection{Definition of Cycle Rate}
\label{subsec:Cyclerate}

The Cycle Rate is the percentage of preference graphs with at least one cycle, indicating inconsistency in pairwise comparisons. For example, if an evaluator produces cycles in 100 out of 164 graphs for the HumanEval dataset, the Cycle Rate is \( \frac{100}{164} \times 100 = 60.97\% \).
A lower Cycle Rate indicates greater consistency, while a higher rate suggests evaluator biases or difficulties with ambiguous comparisons. This metric helps assess the reliability of evaluators, such as GPT-4-o and GPT-4-o-mini, across different datasets.

\subsection{Models Used for Ranking}
\label{subsec:ModelForRanking}

We evaluate model ranking using 30 widely used models from the AlpacaEval dataset, covering diverse architectures and capabilities. Our model set $\mathcal{M}$ includes families such as Llama~\citep{touvron2023llama}, Vicuna~\citep{Zheng2023JudgingLW}, GPT~\citep{achiam2023gpt}, Claude~\citep{claude3news}, Qwen~\citep{yang2024qwen2}, Mistral~\citep{jiang2023mistral}, Yi~\citep{young2024yi}, and WizardLM~\citep{xu2023wizardlm}, ensuring a broad assessment.
Proprietary models include OpenAI’s GPT series (e.g., \texttt{gpt-3.5-turbo-0301}, \texttt{gpt-4o-2024-05-13}) and Anthropic’s Claude models (e.g., \texttt{claude-2}, \texttt{claude-3-opus-20240229}), serving as strong baselines. Open-source models include multiple versions of Llama (e.g., \texttt{Meta-Llama-3-8B-Instruct}, \texttt{llama-2-70b-chat-hf}) and Qwen (e.g., \texttt{Qwen1.5-72B-Chat}, \texttt{Qwen2-72B}).
We also incorporate Mistral models (e.g., \texttt{Mistral-7B-Instruct-v0.2}, \texttt{Mixtral-8x22B-Instruct-v0.1}), reflecting advances in mixture-of-experts architectures. Additional selections include \texttt{gemini-pro}, \texttt{tulu-2-dpo-70b}, \texttt{oasst-sft-llama-33b}, \texttt{dbrx-instruct}, \texttt{wizardlm-13b}, and \texttt{Yi-34B-Chat}, contributing to ranking diversity.
This diverse model set enables a comprehensive evaluation of ranking accuracy using AlpacaEval’s ground-truth rankings.

\subsection{Dataset}
\label{sec:Dataset}

In this appendix, we provide detailed information about the  datasets used in main text.

\begin{itemize}
    \item {UltraFeedback}~\citep{cui2023ultrafeedback}:  UltraFeedback is a large-scale, fine-grained, diverse preference dataset, used for training powerful reward models and critic models. We collect about 64k prompts from diverse resources (including UltraChat, ShareGPT, Evol-Instruct, TruthfulQA, FalseQA, and FLAN). We then use these prompts to query multiple LLMs (see Table for model lists) and generate 4 different responses for each prompt, resulting in a total of 256k samples.
    \item {HH-RLHF}~\citep{bai2022training}:  The HH-RLHF dataset contains human preference data for training language models to be helpful and harmless, as well as red teaming data to identify harmful model outputs. The preference data includes pairs of chosen and rejected responses, while the red teaming data includes transcripts of adversarial interactions with AI assistants, rated for harmfulness. We strictly follow prior works~\citep{song2024preference,yu2023constructive}  for testing.

    \item{MATH~}\citep{hendrycks2021measuring}:  The MATH dataset consists of 12,500 challenging competition-level math problems, each with a detailed step-by-step solution. It is designed to teach models to generate answer derivations and explanations, aiding in mathematical reasoning. Despite progress in improving accuracy, the dataset highlights the limitations of large Transformer models in solving complex math problems without new algorithmic advancements.
    Due to the high resource cost of using the full test set, we randomly sampled 400 problems from the test set for evaluation.

    \item{GSM8k}~\citep{chen2021evaluating}:  GSM8K (Grade School Math 8K) is a collection of 8.5K high-quality math word problems designed for grade school students. It supports the task of multi-step reasoning and question answering in basic math. The problems require 2 to 8 steps, focusing on elementary arithmetic operations (addition, subtraction, multiplication, and division). The solutions are provided in natural language, making it accessible for evaluation of language models' internal reasoning. 
    GSM8K has been widely used to test logic and mathematical capabilities in language models, especially for benchmarks like the LLM Leaderboard.
    Due to the high computational cost of using the entire test set, we randomly sampled 400 data points from the test set for our evaluation.

    \item{GAIA}~\citep{mialon2023gaia}:  The GAIA dataset is a benchmark designed to evaluate next-generation LLMs with augmented capabilities like tooling and search access. It consists of over 450 complex questions with unambiguous answers, requiring various levels of autonomy and tooling. The dataset is divided into three levels, each increasing in difficulty, with a public dev set for validation and a private test set for evaluation. We used the entire test set for our evaluation.

    \item{HumanEval}~\citep{chen2021evaluating}:  The OpenAI HumanEval dataset contains 164 programming problems, each with a function signature, docstring, body, and unit tests. These problems are handwritten to ensure they were not included in the training sets of code generation models. The dataset is designed to evaluate the performance of models in Python code generation.
    We used the entire test set for evaluation.

    \item{AlpacaEval.}~\citep{dubois2024alpacafarm}: AlpacaEval consists of 805 instructions, including 252 from the self-instruct test set~\citep{wang2022self}, 188 from the Open Assistant (OASST) test set, 129 from Anthropic's helpful test set~\citep{zhou2023lima}, 80 from the Vicuna test set~\citep{vicuna2023}, and 156 from the Koala test set~\citep{vu2023koala}.

    \item{LIMA.}~\citep{zhou2023lima}: LIMA collects a training dataset of 1000 prompts and responses, curated to ensure stylistic consistency while accommodating diverse input types. It also includes an open-source test set of 300 prompts and a development set of 50. The data is primarily sourced from community-driven Q\&A platforms like Stack Exchange, wikiHow, and the Pushshift Reddit Dataset~\citep{baumgartner2020pushshift}, along with manually curated examples. 
    The inclusion of human-authored examples further increases dataset diversity. In our experiments, we use the LIMA test set to evaluate our models.

    \item{Vicuna.}~\citep{vicuna2023}: Vicuna organizes its 80 test instructions into eight distinct categories: Fermi problems, commonsense, roleplaying, coding/math/writing tasks, counterfactuals, knowledge, and general questions. This categorization aims to comprehensively assess different facets of chatbot performance. Prior work suggests that Vicuna's instructions are generally of lower complexity and difficulty~\citep{xu2023wizardlm}. We utilize the Vicuna test set to assess the performance of large language models across these varied categories of instructions.

    \item{Self-Instruct.}~\citep{wang2022self}: Self-Instruct contains 252 human-authored test instructions, each paired with a well-constructed output. This dataset is curated to simulate real-world use cases of instruction-following models, spanning various domains such as email composition, social media, productivity tools, and coding tasks. The instructions differ in format and complexity, featuring diverse task lengths and output types such as bullet points, tables, code snippets, and mathematical expressions. In our research, we utilized the Self-Instruct test set to rigorously evaluate our model's ability to follow detailed instructions across multiple domains.

    \item{Wizardlm.}~\citep{xu2023wizardlm}: Wizardlm consists of a training set of 70k examples derived from 52k instructions initially provided by Alpaca. 
    The test set includes 218 instructions sourced from various open-source projects and online communities, covering 29 distinct skills derived from real-world tasks. 
    These skills range from Code Generation \& Debugging to Reasoning, Mathematics, Writing, Complex Format Handling, and Mastery of Extensive Domains. 
    In our study, we employed the Wizardlm test set to evaluate the model's ability to adhere to detailed instructions comprehensively.

    \item{Koala.}~\citep{vu2023koala}: Koala comprises 180 real-world user queries sourced from the web, spanning diverse topics and typically reflecting a conversational tone. These queries are especially relevant for evaluating models intended for chat-based applications. To ensure no overlap with training data, any query yielding a BLEU score above 20\% compared to examples from our training set is excluded. Additionally, queries involving programming or non-English languages are omitted, as our evaluation team, composed of crowd-sourced raters, lacks the expertise to assess such content effectively. We exclusively used the Koala test set to gauge our model's proficiency in handling authentic conversational queries.
\end{itemize}

\subsection{Aggregation Process in \OURS Across Different Tasks}
\label{subsec:aggregation_process}

The \OURS implementation differs between response selection and model ranking tasks due to their objectives.
In {response selection}, \OURS aggregates preference graphs from multiple evaluators for each question into a single DAG. From this DAG, a final ranking is derived, and the top-ranked answer is selected as the output. Notably, aggregation is performed only at the per-question level, without a rank aggregation across questions.
In {model ranking}, \OURS involves two steps: first, aggregating the evaluators' preference graphs into a DAG for each question to rank the models, and second, employing a rank ensemble method to aggregate these per-question rankings into a final overall model ranking across all questions.
For details on the aggregation modeling, see Section~\ref{subsec: application}, covering Response Selection, Model Ranking, and Model Alignment.

\section{{Ablation}  study for Test Time scaling}

We evaluate the impact of removing the ensembling step in \OURS{}, referred to as the \textit{(w/o ensemble)} variant. In this case, individual evaluators’ preference graphs are denoised and converted to rankings, which are then aggregated using methods such as \textit{Weight Score}, \textit{Kemeny}, \textit{Weighted Kemeny}, \textit{Pairwise Majority}, and \textit{Weighted Pairwise Majority} (detailed in  Appendix~\ref{sec:rank_ensemble_method}).
For simplicity of presentation, we use \textit{Weight Score} to represent \textit{\OURS (w/o ensemble) (Weight Score)}. 
As shown in Figure~\ref{fig:ablation}, all \textit{(w/o ensemble)} methods consistently underperform compared to \OURS{}. This performance gap arises because converting graphs to ranks before aggregation leads to information loss. In contrast, \OURS{} ensembles the graphs directly, preserving more detailed preference information and resulting in better final rankings.

\begin{figure}[t]
\centering
\includegraphics[width=1.0\linewidth]{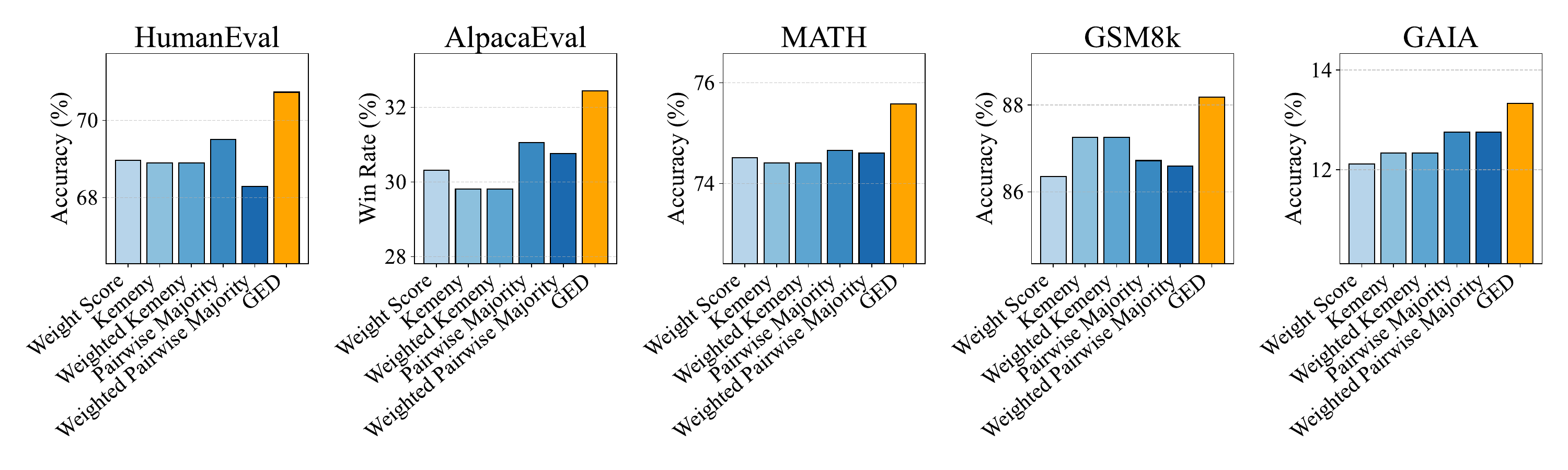}
\captionsetup{skip=0pt} 
\caption{
Comparison of \OURS{} and \textit{(w/o ensemble)} variants. 
\OURS{} outperforms due to preserving more information by directly ensembling preference graphs, while rank aggregation in the \textit{(w/o ensemble)} methods leads to performance loss.
}
\label{fig:ablation}
\end{figure}

\section{Limitation}
\label{sec:limit}
Our framework improves consistency in preference evaluations, but it implicitly promotes a "majority view" by aggregating judgments from multiple LLM-based evaluators. This may inadvertently suppress minority or dissenting preferences, especially in cases where subjectivity, cultural variation, or ethical nuance plays a critical role. While the goal is to reduce noise, there is a risk that important human disagreement is treated as inconsistency to be "denoised." In socially sensitive domains, such as education, healthcare, or content moderation, this could limit the diversity of acceptable outputs or reinforce dominant norms embedded in current LLMs. 

\begin{table}[t]
\centering
\caption{Performance comparison on nuanced quality metrics (\%). \OURS outperforms individual evaluators and random selection across Factuality, Relevance, Coherence, Inform., Helpful. and Validity metrics.}
\label{table:nuanced_metrics}
\scalebox{0.75}{
\begin{tabular}{lccccccc}
\toprule
\textbf{Method} & \textbf{Fact.} & \textbf{Rel.} & \textbf{Coh.} & \textbf{Info.} & \textbf{Help.} & \textbf{Valid.} & \textbf{Avg.} \\
\midrule
Random     & 86.73 & 87.91 & 92.47 & 77.62 & 17.48 & 48.92 & 68.52 \\
Llama3-8B  & 88.59 & 89.91 & 94.41 & 79.77 & 18.48 & 50.92 & 70.35 \\
Mistral-7B & 89.10 & 90.29 & 94.85 & 79.95 & 18.55 & 51.13 & 70.65 \\
Qwen2-7B   & 89.25 & 90.44 & 95.03 & 80.09 & 18.58 & 51.21 & 70.77 \\
\OURS      & \textbf{94.73} & \textbf{95.91} & \textbf{97.36} & \textbf{86.62} & \textbf{19.48} & \textbf{55.92} & \textbf{75.00} \\
\bottomrule
\end{tabular}
}
\end{table}

\section{Evaluating \OURS \   on More Metrics}

To assess \OURS on more metrics, we evaluated its impact on nuanced quality aspects of LLM outputs, including factuality, relevance, coherence, informativeness, helpfulness, and validity. Since LLM tasks often require subtle judgments, \OURS’s adaptability is crucial.
Following~\citep{Chen2023BeyondFA}, we used {Llama3-70B} to generate ten candidate responses per query and applied \OURS for response selection, with {Llama3-8B}, {Mistral-7B}, and {Qwen2-7B} as evaluators. Table~\ref{table:nuanced_metrics} shows that \OURS consistently outperformed individual evaluators and random selection across all metrics. It improved factuality by ~5 \% points over the best evaluator, while also enhancing relevance, coherence, informativeness, and helpfulness.
These results highlight \OURS’s ability to aggregate preferences effectively, capturing subtle qualities of generated content and improving overall reliability and utility in LLM assessments.

\section{Alternative Evaluator Configurations for Model Ranking}
\label{subsec:alternative_evaluators}

To address concerns about the computational cost and fairness of using 70B-level models as preference evaluators in the model ranking task, we conducted additional experiments using smaller, more comparable models. 
Table~\ref{table:model_ranking_small_evaluators} summarizes the performance of these models, both individually and when combined using \OURS.
The results show that \OURS outperforms individual evaluators even when using smaller 7B-scale models, achieving an average score of 62.70 compared to the best individual performance of 57.92 (Qwen2-7B with graph denoising). This demonstrates that the aggregation of smaller models through \OURS effectively enhances performance while reducing computational costs. These findings validate the versatility of \OURS, showing that it can provide robust and accurate rankings without relying solely on large-scale models.

\begin{table*}[h]
\caption{Performance comparison in the model ranking task using 7B-scale models as evaluators on the AlpacaEval dataset. Higher values indicate better performance. \OURS achieves robust performance even with smaller evaluators.}
\centering
\scalebox{0.83}{
\begin{tabular}{@{}llcccccc@{}}
\toprule
\textbf{Model}                                &                      & \textbf{Weight Score} & \textbf{Kemeny} & \textbf{\begin{tabular}[c]{@{}c@{}}Weighted\\ Kemeny\end{tabular}} & \textbf{\begin{tabular}[c]{@{}c@{}}Pairwise\\ Majority\end{tabular}} & \textbf{\begin{tabular}[c]{@{}c@{}}Weighted\\ Pairwise\\ Majority\end{tabular}} & \textbf{Avg.} \\ \midrule
                                              & Llama3-8B            & 35.88                 & 45.80           & 45.80                                                              & 47.23                                                                & 46.85                                                                           & 44.31         \\
                                              & with graph denoising & 37.44                 & 47.54           & 47.54                                                              & 48.92                                                                & 48.18                                                                           & 45.92         \\ \cmidrule(l){2-8} 
                                              & Qwen2-7B             & 55.34                 & 52.87           & 52.87                                                              & 56.05                                                                & 56.59                                                                           & 54.74         \\
                                              & with graph denoising & 56.05                 & 57.43           & 57.43                                                              & 59.32                                                                & 59.41                                                                           & 57.92         \\ \cmidrule(l){2-8} 
                                              & Mistral-7B           & 49.90                 & 53.74           & 53.74                                                              & 58.06                                                                & 57.87                                                                           & 54.66         \\
\multirow{-6}{*}{\textbf{Single evaluator}} & with graph denoising & 50.47                 & 55.06           & 54.92                                                              & 61.39                                                                & 61.21                                                                           & 56.61         \\ \midrule
\multirow{-1}{*}{\textbf{Multiple evaluator}} & \OURS                  & \textbf{57.59}        & \textbf{61.14}  & \textbf{61.14}                                                     & \textbf{67.17}                                                       & \textbf{66.46}                                                                  & \textbf{62.70} \\ \bottomrule
\end{tabular}
}
\label{table:model_ranking_small_evaluators}
\end{table*}

\begin{algorithm}[H]
\caption{Construction of the Preference Graph for Model Ranking}
\label{alg:preference_graph_construction}
\scalebox{0.75}{
\begin{minipage}{1.1\linewidth} 
\begin{algorithmic}[1]
\Require Set of models \( M = \{\mathcal{M}_1, \mathcal{M}_2, \ldots, \mathcal{M}_n\} \), set of questions \( Q = \{q_1, q_2, \ldots, q_t\} \), set of evaluators \( \mathcal{A} = \{a_1, a_2, \ldots, a_k\} \)
\Ensure Set of preference graph sets \(  \{G_a : a \in \mathcal{A}\}  \) for each question \( q \in Q \)
\For{each question \( q \in Q \)}
    \For{each evaluator \( a \in \mathcal{A} \)}
        \State Initialize vertex set \( V_q = \{v_1, v_2, \ldots, v_n\} \), where each \( v_i \) corresponds to model \( \mathcal{M}_i \)
        \State Initialize edge set \( A_a = \emptyset \), and weight function \( w_a: A_a \rightarrow \mathbb{R}^+ \)
        \For{each pair of models \( (\mathcal{M}_i, \mathcal{M}_j) \) with \( i \neq j \)}
            \State Let \( ans_i = \mathcal{M}_i(q) \) and \( ans_j = \mathcal{M}_j(q) \)
            \If{\( a(ans_i, ans_j) > 0 \)} \Comment{evaluator \( a \) prefers \( \mathcal{M}_i \) over \( \mathcal{M}_j \)}
                \If{\( (v_i, v_j) \notin A_a \)}
                    \State Add directed edge \( (v_i \rightarrow v_j) \) to \( A_a \)
                    \State Set \( w_a(v_i, v_j) = 1 \)
                \Else
                    \State Increment \( w_a(v_i, v_j) \) by 1
                \EndIf
            \Else
                \If{\( (v_j, v_i) \notin A_a \)}
                    \State Add directed edge \( (v_j \rightarrow v_i) \) to \( A_a \)
                    \State Set \( w_a(v_j, v_i) = 1 \)
                \Else
                    \State Increment \( w_a(v_j, v_i) \) by 1
                \EndIf
            \EndIf
        \EndFor
        \State Store the preference graph \( G_a = (V_q, A_a, w_a) \) for evaluator \( a \)
    \EndFor
\EndFor
\end{algorithmic}
\end{minipage}
}
\end{algorithm}

\section{Comparison with Point-Wise Scoring Methods}
\label{subsec:pointwise_comparison}

To further investigate the difference between point-wise and preference-based scoring methods, we conducted additional experiments under the response selection setting. 
Specifically, we evaluated a point-wise baseline where evaluators assigned scores to responses on a scale of 1 to 5, selecting the highest-rated response. The results are shown in Table~\ref{tab:pointwise_vs_ged}.
\begin{table*}[h]
\caption{Comparison of point-wise methods and \OURS on response selection tasks.}
\captionsetup{skip=0pt}
\centering
\scalebox{0.85}{
\begin{tabular}{@{}lccccc|c@{}}
\toprule
\textbf{Model} & \textbf{HumanEval} & \textbf{AlpacaEval} & \textbf{MATH} & \textbf{GSM8k} & \textbf{GAIA} & \textbf{Avg} \\
\midrule
Point-wise - Llama3-8B & 60.97 & 28.77 & 74.06 & 82.65 & 10.97 & 51.48 \\
Point-wise - Mistral-7B & 62.83 & 26.93 & 74.02 & 83.21 & 10.22 & 51.44 \\
Point-wise - Qwen2-7B & 60.41 & 28.30 & 74.32 & 84.27 & 10.82 & 51.62 \\
Point-wise - Majority Voting & 63.72 & 29.42 & 74.71 & 84.36 & 11.32 & 52.70 \\
\OURS (Llama3-8B, Mistral-7B, Qwen2-7B) & 70.73 & 32.44 & 75.58 & 88.18 & 13.33 & 56.05 \\
\bottomrule
\end{tabular}
}
\label{tab:pointwise_vs_ged}
\end{table*}
From the results, we observe that point-wise methods consistently underperform compared to \OURS across all tasks. 
This can be attributed to the fact that point-wise scoring assesses each response independently, 
without considering the relative quality of different responses. 
As a result, it lacks the global information available in preference-based ranking, where responses are directly compared. 
While majority voting improves point-wise performance slightly by aggregating scores from multiple evaluators, it remains inferior to \OURS, which leverages pairwise preferences to construct a more informed global ranking.
Despite its limitations, the point-wise approach has an advantage in terms of computational efficiency, as it requires fewer comparisons than pairwise methods. 
This trade-off suggests that the choice between point-wise and preference-based scoring should depend on the specific application: point-wise scoring may be preferable in large-scale settings where efficiency is critical, while \OURS is more effective for achieving higher-quality rankings.

\section{Impact of Evaluator Quantity on Denoising Quality}
\label{sec:evaluator_quantity}

We investigated how the number of evaluators affects the denoising quality in \OURS by conducting experiments using different numbers of evaluators in the response selection setting. Specifically, we evaluated the performance of \OURS when aggregating preferences from two, three, and four evaluators. The evaluators used were Llama3-8B, Mistral-7B, Qwen2-7B, and Gemma-9B.

\begin{table*}[h]
\caption{Performance comparison of \OURS with varying numbers of evaluators across five benchmarks. Increasing the number of evaluators enhances denoising quality, reflected in improved performance metrics.}
\captionsetup{skip=0pt}
\centering
\scalebox{0.82}{
\begin{tabular}{@{}lcccccc@{}}
\toprule
\textbf{Evaluators Set} & \textbf{HumanEval} & \textbf{AlpacaEval} & \textbf{MATH}  & \textbf{GSM8k} & \textbf{GAIA}  & \textbf{Avg.} \\
\midrule
Llama3-8B, Mistral-7B & 69.21 & 31.87 & 74.97 & 86.92 & 12.51 & 55.10 \\ 
Llama3-8B, Mistral-7B, Qwen2-7B & 70.73 & 32.44 & 75.58 & 88.18 & 13.33 & 56.05 \\ 
Llama3-8B, Mistral-7B, Qwen2-7B, Gemma-9B & {70.98} & {32.87} & {75.91} & {88.75} & {13.46} & {56.39} \\ 
\bottomrule
\end{tabular}
}
\label{table:evaluator_quantity}
\end{table*}

As shown in Table~\ref{table:evaluator_quantity}, increasing the number of evaluators consistently improves the denoising quality of \OURS, as reflected by higher performance across all benchmarks. For instance, the average performance improves from 55.10\% with two evaluators to 56.39\% with four evaluators. This enhancement can be attributed to the diverse perspectives and complementary strengths of multiple evaluators, contributing to a more robust and accurate aggregation of preferences.
These results highlight the importance of selecting a diverse and capable set of evaluators to enhance \OURS's effectiveness. Incorporating more evaluators allows the denoising process to better identify and mitigate inconsistencies in the preference graphs, leading to improved overall performance in response selection tasks.

\section{Construction of the Preference Graph}
\label{sec:preferenceConstruct}
In this section, we provide a detailed explanation of the process used to construct the preference graph sets for both the model ranking and response selection tasks, as outlined in Algorithms~\ref{alg:preference_graph_construction} and \ref{alg:preference_graph_gen_response_selection}. These algorithms form the backbone of our method, enabling the representation of pairwise preferences as directed graphs, which are essential for downstream aggregation and ranking.

Algorithm~\ref{alg:preference_graph_construction} describes the construction process for generating a set of preference graphs for the model ranking task. The procedure is as follows:
\textit{Initialization}: For each question \( q \in Q \), we begin by initializing a vertex set \( V_q \), where each vertex \( v_i \) corresponds to a model \( \mathcal{M}_i \) in the set of models \( M \). We also initialize an empty set of edges \( A_a \) and a weight function \( w_a \), which will be used to track the strength of the preferences between model pairs.
\textit{Pairwise Comparisons}: For each pair of models \( \mathcal{M}_i \) and \( \mathcal{M}_j \), the assigned evaluator \( a \in \mathcal{A} \) assesses their responses to the given question \( q \). If evaluator \( a \) prefers \( \mathcal{M}_i \) over \( \mathcal{M}_j \), a directed edge \( (v_i \rightarrow v_j) \) is added to the edge set \( A_a \), and its corresponding weight is incremented. Conversely, if \( \mathcal{M}_j \) is preferred, the edge \( (v_j \rightarrow v_i) \) is added or its weight is updated.
\textit{Graph Storage}: Once all pairwise comparisons have been processed for a given evaluator, the resulting graph \( G_a = (V_q, A_a, w_a) \) is stored for that evaluator. This process is repeated for all evaluators in \( \mathcal{A} \) and for all questions in \( Q \), generating a set of preference graphs for each evaluator and question.

Algorithm~\ref{alg:preference_graph_gen_response_selection} follows a similar structure but applies to the response selection task, where the objective is to rank a set of candidate answers for each question:
\textit{Initialization}: For each question \( q \in Q \), we initialize a vertex set \( V_q \), where each vertex corresponds to a candidate answer \( ans_i \). As in the model ranking task, we also initialize an edge set \( A_a \) and a weight function \( w_a \) for each evaluator \( a \in \mathcal{A} \).
\textit{Pairwise Comparisons}: Evaluators compare the quality of pairs of candidate answers \( ans_i \) and \( ans_j \) for each question. A directed edge is added based on the evaluator’s preference, with the weight reflecting the strength of preference. As before, if evaluator \( a \) prefers \( ans_i \) over \( ans_j \), an edge \( (v_i \rightarrow v_j) \) is added or its weight incremented, and vice versa.
\textit{Graph Storage}: After all pairwise comparisons are complete, the preference graph \( G_a = (V_q, A_a, w_a) \) for evaluator \( a \) is stored. This procedure is repeated for all evaluators and questions, resulting in a set of preference graphs for each evaluator and each question.

Both algorithms ensure that the preference graphs are constructed in a consistent manner, forming the basis for the aggregation and denoising processes used later in our framework. 
These graphs encapsulate the evaluators’ preferences and provide a structured representation of pairwise comparisons, facilitating further analysis.

\begin{algorithm}[H]
\caption{Construction of the Preference Graph for Response Selection}
\label{alg:preference_graph_gen_response_selection}
\scalebox{0.80}{
\begin{minipage}{1.1\linewidth} 
\begin{algorithmic}
\Require Set of candidate answers \( \{ans_1, ans_2, \ldots, ans_n\} \) for each question \( q \in Q \), set of evaluators \( \mathcal{A} = \{a_1, a_2, \ldots, a_k\} \)
\Ensure Set of preference graph sets \( \{G_a : a \in \mathcal{A}\} \) for each question \( q \in Q \)
\For{each question \( q \in Q \)}
    \For{each evaluator \( a \in \mathcal{A} \)}
        \State Initialize vertex set \( V_q = \{v_1, v_2, \ldots, v_n\} \), where each \( v_i \) corresponds to \( ans_i \)
        \State Initialize edge set \( A_a = \emptyset \), and weight function \( w_a: A_a \rightarrow \mathbb{R}^+ \)
        \For{each pair of answers \( (ans_i, ans_j) \) with \( i \neq j \)}
            \If{\( a(ans_i, ans_j) > 0 \)} 
                \If{\( (v_i, v_j) \notin A_a \)}
                    \State Add directed edge \( (v_i \rightarrow v_j) \) to \( A_a \), set \( w_a(v_i, v_j) = 1 \)
                \Else
                    \State Increment \( w_a(v_i, v_j) \) by 1
                \EndIf
            \Else
                \If{\( (v_j, v_i) \notin A_a \)}
                    \State Add directed edge \( (v_j \rightarrow v_i) \) to \( A_a \), set \( w_a(v_j, v_i) = 1 \)
                \Else
                    \State Increment \( w_a(v_j, v_i) \) by 1
                \EndIf
            \EndIf
        \EndFor
        \State Store the preference graph \( G_a = (V_q, A_a, w_a) \) for evaluator \( a \)
    \EndFor
\EndFor
\end{algorithmic}
\end{minipage}
}
\end{algorithm}

\section{The Importance of Addressing Cyclic Inconsistency}
\label{app:cyclic_inconsistency}

Cyclic inconsistencies in preference data introduce contradictions that undermine meaningful ranking and evaluation. 
Many applications, including alignment optimization~\citep{ouyang2022training, song2024preference}, recommendation systems~\citep{hou2024large}, and model evaluation~\citep{liu2024aligning}, rely on transitive preferences for stability and interpretability. 
For example, OpenAI's RLHF pipelines~\citep{ouyang2022training} construct training data based on transitive preferences (e.g., A $\succ$ B $\succ$ C), which ensures consistency in model learning. However, if preference cycles exist (e.g., A $\succ$ B $\succ$ C $\succ$ A), the resulting contradictions lead to instability and degrade training performance.
In our experiments, removing cycles and transforming preference graphs into directed acyclic graphs (DAGs) improved results across multiple tasks, including response selection, model ranking, and instruction tuning. These improvements were consistent across widely used benchmarks, suggesting that most observed cycles are artifacts of noise rather than meaningful violations of transitivity. Ensuring acyclicity enhances dataset reliability and improves model effectiveness in ranking and evaluation tasks.

\section{Proof of Theorem \ref{theo:ensemble_first}}
\label{sec:proof}

\ensembleFirst*

\begin{proof}
    For brevity, we consider all edges in $G$ have weights equal to 1 and all weights in $\widehat{G}$ are divided by $N$. By construction, we can notice that for each $(u\ra v)\in A$, the weight $w_{\widehat{G}}(v\ra u)$ can be viewed as an empirical estimate of $\delta_1$. Then, we claim that the following two events can imply $G\subseteq\MAS(\widehat{G})$:
    \begin{itemize}
        \item $(\mc{E}_1)$ For any $(u\ra v)\in A$, it holds $\abs{w_{\wg}(v\ra u)-\delta_1}\leq\frac{\epsilon}{2}$.

        \item $(\mc{E}_2)$ For any pair of nodes $(u, v)$ such that $(u\ra v), (v\ra u)\notin A$, it holds $\abs{w_{\wg}(u\ra v)-w_{\wg}(v\ra u)}\leq\frac{\epsilon}{U}$.
    \end{itemize}
    To see this, first by Lemma \ref{lem:mag_double_edge}, we know that for any pair of nodes $(u, v)$, $\MAS(\widehat{G})$ will contain \textit{exactly one} of $(u\ra v)$ and $(v\ra u)$.\footnote{There is non-zero probability that some edges in $\wg$ will have zero weight, but we treat them as existing for the ease of argument. That is, we allow only $\wg$ to contain zero-weight edges.} Therefore, for any $(u\ra v)\in A$, $\MAS(\wg)$ will contain exactly one of $(u\ra v)$ and $(v\ra u)$. Then, since $\mc{E}_1$ holds, $\delta_1 = 0.5-\epsilon < 0.5$ and $w_{\wg}(u\ra v) + w_{\wg}(v\ra u) = 1$, we have $w_{\wg}(u\ra v) - w_{\wg}(v\ra u) \geq \epsilon$. Furthermore, since $\mc{E}_2$ holds, for $(u, v)$ such that $(u\ra v), (v\ra u)\notin A$, arbitrary way of edge removing among these nodes can influence the total edge weights by at most $\epsilon$. Therefore, when applying the denoising operation to $\wg$, for any $(u\ra v)\in A$, only $(u\ra v)$ will be kept in $\MAS(\wg)$, which makes $G\subseteq\MAS(\wg)$. As a result, we have $\mathbb{P}(\mc{E}_1\cap\mc{E}_2)\leq \mathbb{P}\Sp{G\subseteq\MAS(\wg)}$.

    Then, we can now bound the probability of $\mc{E}_1\cap\mc{E}_2$. In particular, for fixed $(u\ra v)\in A$, since $w_{\widehat{G}}(v\ra u)$ is an empirical mean estimate of $\delta_1$, by Hoeffding's inequality, we have
    \begin{align*}
        \mathbb{P}\Sp{\abs{w_{\wg}(v\ra u)-\delta_1}\leq\frac{\epsilon}{2}}\geq& 1- 2\exp\Sp{-N\epsilon^2/2}\\
        \implies \mathbb{P}(\mc{E}_1)\geq& 1 - 2|A|\exp\Sp{-N\epsilon^2/2},
    \end{align*}
    where the second inequality comes from the union bound over all edges in $A$. Similarly, for fixed node pair $(u, v)$ that is unconnected in $G$, $w_{\wg}(u\ra v)-w_{\wg}(v\ra u)$ can be viewed as $\frac{1}{N}\sum_{i=1}^{N}X_i$, where $X_i$'s are i.i.d. and
    $$X_i=\begin{cases}
        1,&\text{with probability }\frac{\delta_2}{2}\\
        -1,& \text{with probability }\frac{\delta_2}{2}\\
        0,& \text{with probability } 1-\delta_2
    \end{cases}.$$
    Therefore, by Bernstein's inequality, we have
    \begin{align*}
        & \mathbb{P}\Sp{\abs{w_{\wg}(u\ra v)-w_{\wg}(v\ra u)}\leq\frac{\epsilon}{U}} \\
        & \geq1 - 2\exp\Sp{-\frac{N\epsilon^2}{6U^2\delta_2 + 2U\epsilon}}\\
        \implies \mathbb{P}(\mc{E}_2)
        &\geq  1 - 2U\exp\Sp{-\frac{N\epsilon^2}{6U^2\delta_2 + 2U\epsilon}},
    \end{align*}
    where the second inequality is an union bound over all unconnected node pairs in $G$. As a result, we eventually have
    \begin{align*}
        & \mathbb{P}\Sp{G\subseteq\MAS(\wg)}\geq\mathbb{P}\Sp{\mc{E}_1\cap\mc{E}_2} \\
         & \geq 1 - 2|A|\exp\Sp{-\frac{N\epsilon^2}{2}} - 2U\exp\Sp{-\frac{N\epsilon^2}{6U^2\delta_2 + 2U\epsilon}}.
    \end{align*}
\end{proof}

\begin{lemma}
    \label{lem:mag_double_edge}
    For a weighted directed graph $G=(V, A, w)$, if $(u\ra v), (v\ra u)\in A$, then $\MAS(G)$ contains exactly one of $(u\ra v)$ and $(v\ra u)$.
\end{lemma}
\begin{proof}
    Recall that $\MAS(G)$ gives an acyclic subgraph of $G$ with the maximum weight. Since it has to be acyclic, it is obvious that $\MAS(G)$ cannot contain both $(u\ra v)$ and $(v\ra u)$. 
    
    We will then use contradiction to show it is impossible for $\MAS(G)$ to contain neither $(u\ra v)$ nor $(v\ra u)$. Suppose this is true instead. Then, since $\MAS(G)$ is a maximum acyclic subgraph, adding either $(u\ra v)$ or $(v\ra u)$ to $\MAS(G)$ will make it cyclic. That is, $\MAS(G)$ contains a path that goes from $v$ to $u$; meanwhile, it also contains a path from $u$ to $v$. As a result, it contains a cycle that goes from $v$ to $u$ and then from $u$ to $v$, which contradicts with the fact that $\MAS(G)$ is a maximum acyclic subgraph. Therefore, $\MAS(G)$ must contain exactly one of $(u\ra v)$ and $(v\ra u)$.
\end{proof}

\section{Impact of Evaluator Weighting on \OURS}
\label{sec:ged_weighting}

In our theoretical analysis, \OURS assumes equal weighting of edges in the preference graphs. However, in practical scenarios, evaluators may have varying levels of reliability or expertise. Incorporating evaluator-specific confidence scores or performance metrics could enhance the effectiveness of \OURS. To investigate this, we conducted experiments using a weighted version of \OURS, referred to as {WeightGED}, under the Response Selection setting.
In {WeightGED}, we assigned weights to evaluators based on their individual performance on specific datasets. For example, on the GSM8k dataset, the response selection accuracies for Llama3-8B, Mistral-7B, and Qwen2-7B were 62.19\%, 67.24\%, and 61.58\%, respectively. These accuracies were normalized to compute evaluator weights:

\begin{align*}
\text{weight}(\text{Llama3-8B}) &= \frac{62.19}{62.19 + 67.24 + 61.58} = 0.326, \\
\text{weight}(\text{Mistral-7B}) &= \frac{67.24}{62.19 + 67.24 + 61.58} = 0.352, \\
\text{weight}(\text{Qwen2-7B}) &= \frac{61.58}{62.19 + 67.24 + 61.58} = 0.322.
\end{align*}

These weights were used to scale the contributions of each evaluator's preferences during graph construction. We compared the performance of \OURS and WeightGED across multiple benchmarks, as presented in Table~\ref{table:weightged_comparison}.

\begin{table*}[t]
\caption{Performance comparison of \OURS and WeightGED using 7B-scale evaluators across five benchmarks. WeightGED marginally outperforms \OURS, demonstrating the potential benefits of incorporating evaluator-specific weights.}
\centering
\scalebox{0.85}{
\begin{tabular}{@{}llcccccc@{}}
\toprule
\multicolumn{2}{l}{\textbf{Method}}          & \textbf{HumanEval} & \textbf{AlpacaEval} & \textbf{MATH}  & \textbf{GSM8k} & \textbf{GAIA}  & \textbf{Avg.}   \\ \midrule
                                              & \OURS                & 70.73               & 32.44          & 75.58          & 88.18          & 13.33           & 56.05           \\ \cmidrule(l){2-8} 
\multirow{-2}{*}{\textbf{7B Evaluators}}      & WeightGED          & \textbf{70.97}      & \textbf{32.67} & \textbf{75.71} & \textbf{88.24} & \textbf{13.56}  & \textbf{56.23}  \\ \midrule
                                              & \OURS                & 73.21               & 59.87          & 82.49          & 86.43          & 16.27           & 63.65           \\ \cmidrule(l){2-8} 
\multirow{-2}{*}{\textbf{GPT Evaluators}}     & WeightGED          & \textbf{74.52}      & \textbf{61.71} & \textbf{83.93} & \textbf{87.84} & \textbf{17.32}  & \textbf{65.06}  \\ \bottomrule
\end{tabular}
}
\label{table:weightged_comparison}
\end{table*}

As shown in Table~\ref{table:weightged_comparison}, WeightGED achieves marginal but consistent improvements over \OURS across all benchmarks. For instance, using 7B-scale evaluators, WeightGED improves the average performance from 56.05\% to 56.23\%. Similarly, with GPT evaluators, the average performance increases from 63.65\% to 65.06\%. These results suggest that incorporating evaluator-specific weights based on performance metrics can enhance the effectiveness of \OURS.
Furthermore, we conducted additional experiments using stronger evaluators such as GPT-3.5, GPT-4-o-mini, and GPT-4-o. The weights were computed based on their respective performance accuracies, following the same normalization procedure. The improvements observed with these evaluators reinforce the potential of weighting schemes in enhancing \OURS.
In summary, our preliminary findings indicate that integrating evaluator-specific confidence scores or performance metrics is a promising direction for future work. Systematically designing and optimizing these weighting schemes could lead to more robust and accurate evaluation frameworks.

\section{Expanded Related Work on Preference Denoising for LLMs}

This work situates itself within the broader field of denoising preference evaluations for LLMs, addressing inconsistencies and noise in preference graphs. We acknowledge that existing literature has explored two primary approaches to this challenge: within-model preference modeling and modular pre-processing. Below, we provide a detailed comparison of \OURS{} with representative methods from these approaches.
Within-model approaches, such as robust DPO (rDPO)~\citep{Chowdhury2024ProvablyRD} and conservative DPO (cDPO)~\citep{mitchell2023note}, focus on integrating denoising mechanisms directly within the preference modeling process. These methods incorporate regularization techniques to mitigate the effects of noisy or adversarial preference data during model alignment. While these approaches are effective in refining the preference modeling pipeline, they are tightly coupled with specific models and tasks, limiting their versatility. In contrast, \OURS{} is a modular framework that operates as a pre-processing step, denoising preference graphs before downstream tasks, making it adaptable to a wide range of applications and models.
Modular approaches like CURATRON~\citep{naresh2024curatron} address noise and missing comparisons in preference datasets using techniques such as low-rank matrix completion. While CURATRON effectively mitigates certain types of noise, it does not explicitly target cyclic inconsistencies (e.g., A $\succ$ B, B $\succ$ C, C $\succ$ A), which are a critical focus of \OURS{}. By leveraging a graph-based framework, \OURS{} detects and removes such cycles through its denoising process, ensuring that the resulting preference graph is acyclic and thus more consistent and reliable for downstream use.
Additionally, \OURS{} distinguishes itself by providing theoretical guarantees for recovering the ground truth preference DAG under certain conditions. Furthermore, the ensemble mechanism in \OURS{} demonstrates the novel insight that combining small evaluators can surpass the performance of a single stronger evaluator, a feature not emphasized in the aforementioned methods.

\section{Theoretical Analysis of Graph Denoising}
\label{sec:proof_wfas_convergence}

Let $G=(V,A,w)$ be the aggregated digraph, with $w:A\to\mathbb{R}_{\ge 0}$.
Write $m\triangleq |\{(u,v)\in A: w(u,v)>0\}|$.
For $u\in V$, define
\[
\begin{aligned}
d^+(u) &\triangleq \sum_{(u,v)\in A} w(u,v), 
d^-(u) \triangleq \sum_{(v,u)\in A} w(v,u), \\
d(u)   &\triangleq d^+(u) + d^-(u), 
\delta(u) \triangleq d^+(u) - d^-(u).
\end{aligned}
\]

Let $s=s_1\Vert s_2$ be the order returned by Algorithm~\ref{alg:pgD}, and
\[
\begin{aligned}
R(s) &\triangleq \{(v_j \to v_i) \in A :\ j > i \text{ in } s\}, \\
G' &\triangleq (V, A \setminus R(s), w).
\end{aligned}
\]

\begin{proposition}[Acyclicity]
\label{prop:acyclicity}
$G'$ is a directed acyclic graph.
\end{proposition}

\begin{proof}
We prove by contradiction.
Assume $G'$ contains a directed cycle $C=(u_1\!\to u_2\!\to\cdots\!\to u_\ell\!\to u_1)$.
Since $G'$ keeps exactly the \emph{forward} arcs w.r.t.\ the total order $s$, every arc $(u_t\!\to u_{t+1})$ in $C$ must satisfy
$\mathrm{pos}_s(u_t)<\mathrm{pos}_s(u_{t+1})$, where $\mathrm{pos}_s(\cdot)$ is the index in $s$.
Summing these strict inequalities around the cycle yields
\[
\mathrm{pos}_s(u_1) < \mathrm{pos}_s(u_2) < \cdots < \mathrm{pos}_s(u_\ell) < \mathrm{pos}_s(u_1),
\]
which is impossible.
Therefore $G'$ has no directed cycle.
\end{proof}

\begin{assumption}[Bounded integral weights]
\label{asm:bounded}
There exists a fixed integer $k\ge 1$ (the number of evaluators) such that
$w(u,v)\in\{0,1,\dots,k\}$ for all $(u,v)\in A$.
\end{assumption}

\noindent
Assumption~\ref{asm:bounded} holds in our ensemble setting because each evaluator contributes at most one unit to any arc weight.

\paragraph{Data structures.}
We maintain:
(i) a queue $Q^{-}$ of sources ($d^-(u)=0$); (ii) a queue $Q^{+}$ of sinks ($d^+(u)=0$);
(iii) $\delta$-buckets $\{B_\Delta\}_{\Delta=-M}^{M}$, where $M\triangleq \max_{u\in V} d(u)$.
Each bucket is a doubly linked list of vertices with current $\delta(u)=\Delta$ and $d^\pm(u)>0$.
We also keep an integer pointer $p$ storing the maximum index $\Delta$ with $B_\Delta\neq\emptyset$.
All vertices are in exactly one of $Q^{-}$, $Q^{+}$, or some $B_\Delta$.

\begin{lemma}[Range and update magnitude]
\label{lem:range}
Under Assumption~\ref{asm:bounded}: (a) $M=\mathcal{O}(m)$; (b) when removing a vertex $u$, for any neighbor $v$ we have
\[
\begin{aligned}
\delta_{\text{new}}(v) &= \delta_{\text{old}}(v) + w(u,v) - w(v,u), \\
\Rightarrow\quad
|\delta_{\text{new}}(v) - \delta_{\text{old}}(v)| &\le k.
\end{aligned}
\]
\end{lemma}

\begin{proof}
(a) Since each incident arc contributes at most $k$ to $d^\pm(\cdot)$,
\[
d(u)=\sum_{(u,x)}w(u,x) + \sum_{(x,u)}w(x,u)
\;\le\; k\cdot \deg(u),
\]
hence $M=\max_u d(u)\le k\cdot \max_u \deg(u)\le k\cdot m=\mathcal{O}(m)$.
(b) When $u$ is removed, exactly the arcs $(u,v)$ and $(v,u)$ (if present) are deleted among $v$'s incidences.
Thus $d^-(v)$ decreases by $w(u,v)$, $d^+(v)$ decreases by $w(v,u)$, and
$\delta(v)=d^+(v)-d^-(v)$ increases by $w(u,v)$ and decreases by $w(v,u)$, as claimed.
By Assumption~\ref{asm:bounded}, $w(u,v),w(v,u)\le k$, hence the $|\cdot|$ bound.
\end{proof}

\begin{lemma}[Bucket maintenance in $\mathcal{O}(1)$ amortized time]
\label{lem:buckets}
Under Assumption~\ref{asm:bounded}
, the following hold.
\begin{enumerate}
\item \textbf{Initialization.} Computing $d^\pm(\cdot)$ and placing vertices into $Q^{-}$, $Q^{+}$, or $B_\Delta$ takes $\mathcal{O}(m)$ time and $\mathcal{O}(m)$ space.
\item \textbf{Updates.} Removing a vertex $u$ and processing all its incident arcs take time $\mathcal{O}(1+\deg(u))$:
every neighbor $v$ is updated once, and $v$ moves from $B_{\Delta}$ to $B_{\Delta'}$ with $|\Delta'-\Delta|\le k$;
enqueuing/dequeuing from $Q^{\pm}$ or $B_\Delta$ is $\mathcal{O}(1)$ via doubly linked lists.
\item \textbf{Selecting $\max \delta$.} Maintaining pointer $p$ to the largest index with $B_p\neq\emptyset$ costs $\mathcal{O}(1)$ amortized per removal.
\end{enumerate}
\end{lemma}

\begin{proof}
(1) We make one pass over adjacency lists: for $(x,y)$ add $w(x,y)$ to $d^+(x)$ and to $d^-(y)$.
This is $\mathcal{O}(m)$.
Classification into $Q^{\pm}$ or some $B_\Delta$ is then $\mathcal{O}(n)$, dominated by $\mathcal{O}(m)$ since a connected digraph has $m\ge n-1$.
(2) When removing $u$, we unlink $u$ from its current container in $\mathcal{O}(1)$.
For each incident arc touching $u$, we visit the unique neighbor $v$ once, adjust $d^\pm(v)$ by Lemma~\ref{lem:range}, recompute $\delta(v)$, and move $v$:
\[
\begin{cases}
\text{to }Q^- & \text{if }d^-(v)=0 \text{ and } d^+(v)>0,\\[2pt]
\text{to }Q^+ & \text{if }d^+(v)=0 \text{ and } d^-(v)>0,\\[2pt]
\text{to }B_{\delta(v)} & \text{if }d^+(v),d^-(v)>0.
\end{cases}
\]
Each move is pointer manipulation in a doubly linked list, costing $\mathcal{O}(1)$.
Hence the work per removal is $\mathcal{O}(1+\deg(u))$ and each arc is processed \emph{exactly once} (when the first endpoint is removed).
(3) \emph{Amortized bound for $p$.}
Let $p_t$ be the pointer after the $t$-th removal.
There are two kinds of motion:

\noindent \emph{Downward motion.}
Whenever $B_{p_t}$ becomes empty, we decrement $p_t$ until finding the next nonempty bucket.
Since $p$ never moves below $-M$ and never increases during this downward scan, the \emph{total} number of decrements over the entire execution is at most $2M=\mathcal{O}(m)$.

\noindent \emph{Upward motion.}
Deleting $u$ can increase some $\delta(v)$, but by Lemma~\ref{lem:range} the new maximum bucket index can exceed the previous maximum by at most $k$:
\[
\begin{aligned}
\max_{x\in V} \delta_t(x)
&\le \max\Bigl\{
    \max_{x\in V} \delta_{t-1}(x),\ 
    \max_{v\in N(u)} \delta_{t-1}(v) + k
\Bigr\} \\
&\le p_{t-1} + k.
\end{aligned}
\]

\noindent We therefore scan upward by at most $k$ to reestablish $p_t$.
Since $k$ is a fixed constant and there are $|V|$ removals, the \emph{total} upward scans cost $\mathcal{O}(k|V|)=\mathcal{O}(|V|)$.
For connected digraphs, $|V|\le m+1$, hence $\mathcal{O}(|V|)=\mathcal{O}(m)$.
Combining both motions gives $\mathcal{O}(1)$ amortized per removal.
\end{proof}

\begin{theorem}[Linear time and space]
\label{thm:linear}
Under Assumption~\ref{asm:bounded}, Algorithm~\ref{alg:pgD} runs in $\mathcal{O}(m)$ time and uses $\mathcal{O}(m)$ space.
\end{theorem}

\begin{proof}
\textbf{Time.}
Initialization is $\mathcal{O}(m)$ by Lemma~\ref{lem:buckets}(1).
Across the execution, each vertex is removed once.
By Lemma~\ref{lem:buckets}(2), the cost of removing $u$ is $\mathcal{O}(1+\deg(u))$.
Summing over all $u$ and noting $\sum_u \deg(u)=2m$, we obtain
\[
\begin{aligned}
\sum_{u \in V} \mathcal{O}(1 + \deg(u))
&= \mathcal{O}\bigl(|V| + \sum_{u} \deg(u)\bigr) \\
&= \mathcal{O}(|V| + m) \\
&= \mathcal{O}(m),
\end{aligned}
\]
since a connected digraph has $|V|\le m+1$.
The amortized $\mathcal{O}(1)$ selection of the maximum-$\delta$ bucket follows from Lemma~\ref{lem:buckets}(3), and is already charged in the removal step.
Therefore the overall running time is $\mathcal{O}(m)$.

\noindent \textbf{Space.}
Adjacency lists and degree/weight arrays take $\mathcal{O}(m)$.
The queues $Q^\pm$ store disjoint subsets of $V$; the buckets $\{B_\Delta\}$ occupy an array of size $2M+1=\mathcal{O}(m)$ with total stored vertices $|V|$.
All pointers and auxiliary counters are $\mathcal{O}(1)$ each.
Hence the working memory is $\mathcal{O}(m)$.
\end{proof}

\begin{remark}[Arbitrary real weights]
If $w$ is real-valued without a constant bound on per-arc increments, the bucket update in Lemma~\ref{lem:buckets} no longer yields $\mathcal{O}(1)$.
Replacing buckets by a binary max-heap over current $\delta(u)$ preserves correctness and yields $\mathcal{O}(m\log n)$ time and $\mathcal{O}(m)$ space.
\end{remark}

\section{Impact of Evaluator Selection on \OURS }
\label{sec:evaluator_selection}

The selection of preference evaluators is a critical factor influencing the performance of \OURS. While our initial experiments demonstrated the benefits of combining preference evaluators, further analysis is necessary to understand how their diversity and capabilities affect \OURS's effectiveness. To address this, we conducted additional experiments evaluating three different configurations of evaluator sets: the original set comprising {Llama3-8B}, {Mistral-7B}, and {Qwen2-7B}; a diverse set that adds {Gemma-9B} to introduce more variation in model architecture and training data; and a higher-capability set that replaces smaller models with larger ones, specifically {Llama3-70B}, {Mistral-8×7B}, and {Qwen2-72B}.
The results, shown in Table~\ref{table:evaluator_selection}, reveal two key insights. First, incorporating diversity by adding Gemma-9B improves \OURS's performance slightly, increasing the average score from 56.05\% to 56.39\%. This suggests that models with diverse training data and architectures can contribute to more robust aggregated evaluations. Second, replacing smaller evaluators with higher-capability models yields a more substantial improvement, with the average score rising to 60.89\%. Notably, benchmarks like AlpacaEval and GAIA benefit the most from the advanced reasoning and language understanding capabilities of larger models.
These findings demonstrate the importance of both diversity and capability in evaluator selection. Diversity provides marginal gains by bringing varied perspectives, while higher-capacity models contribute to more significant improvements by enhancing the overall quality of evaluations. This suggests that, when computational resources allow, incorporating advanced models into the evaluator set can meaningfully boost \OURS’s performance.

\begin{table*}[t]
\caption{Performance comparison with different evaluator sets across five benchmarks. The results highlight the impact of evaluator diversity and capability on \OURS's effectiveness.}
\centering
\scalebox{1.1}{
\begin{tabular}{@{}lcccccc@{}}
\toprule
\textbf{Evaluators Set}   & \textbf{HumanEval} & \textbf{AlpacaEval} & \textbf{MATH}  & \textbf{GSM8k} & \textbf{GAIA}  & \textbf{Avg}   \\ \midrule
Original                  & 70.73              & 32.44               & 75.58          & 88.18          & 13.33          & 56.05          \\ 
Diverse                   & 70.98              & 32.87               & 75.91          & 88.75          & 13.46          & 56.39          \\ 
Higher Capability         & {74.73}     & {47.92}      & {75.58} & {90.04} & {16.21} & {60.89} \\ \bottomrule
\end{tabular}
}
\label{table:evaluator_selection}
\end{table*}

\begin{algorithm}[h]
\caption{ActiveGED }
\label{alg:activeged}
\begin{algorithmic}[1]
\Require Candidate set $V = \{v_1, v_2, \dots, v_n\}$; Initial preference graph $G = (V, A, w)$; Evaluator set $E = \{\text{ev}_1, \dots, \text{ev}_k\}$; Total budget $B$; Random budget ratio $\alpha \in (0, 1)$
\Ensure Updated preference graph $G^* = (V, A^*, w^*)$
\State Randomly select $m_{\text{rand}} = \alpha B$ edges from $(V \times V) \setminus A$ to form $A_{\text{rand}}$
\For{each edge $(u, v) \in A_{\text{rand}}$}
    \For{each evaluator $\text{ev}_i \in E$}
        \State Obtain preference weight $w_i(u, v)$
    \EndFor
    \State Aggregate weights $w(u, v) = \frac{1}{k} \sum_{i=1}^k w_i(u, v)$
\EndFor
\State Update $A^* \leftarrow A \cup A_{\text{rand}}$
\State Set current budget $b \leftarrow m_{\text{rand}}$
\State Compute PageRank $PR(v)$ for all $v \in V$ using $G^* = (V, A^*, w^*)$
\While{$b < B$}
    \For{each unevaluated edge $(u, v) \in (V \times V) \setminus A^*$}
        \State Estimate uncertainty $U(u, v)$ based on current PageRank scores
    \EndFor
    \State Select edge $(u^*, v^*)$ with highest $U(u, v)$
    \For{each evaluator $\text{ev}_i \in E$}
        \State Obtain preference weight $w_i(u^*, v^*)$
    \EndFor
    \State Aggregate weights $w(u^*, v^*) = \frac{1}{k} \sum_{i=1}^k w_i(u^*, v^*)$
    \State Update $A^* \leftarrow A^* \cup \{(u^*, v^*)\}$
    \State Increment $b \leftarrow b + 1$
    \State Recompute PageRank $PR(v)$ for all $v \in V$ using the updated $G^*$
\EndWhile
\State \textbf{return} $G^* = (V, A^*, w^*)$
\end{algorithmic}
\end{algorithm}

\section{Cost Considerations and Active Learning with ActiveGED}
\label{sec:activeged}

Using multiple evaluators and aggregating their preferences can be computationally expensive, especially when constructing dense preference graphs. This limitation becomes more pronounced in scenarios where preferences across all pairs of evaluators are required, which scales quadratically with the number of responses or models being compared. To address this, we clarify the computational trade-offs and propose an active learning-based approach to reduce the number of required pairwise evaluations.
To further reduce the number of pairwise evaluations needed, we developed an active learning algorithm called \textbf{ActiveGED}. This algorithm strategically selects the most informative pairs for evaluation, effectively lowering the overall computational cost while maintaining high performance.  ActiveGED combines random sampling with uncertainty-based selection to maximize information gain from each additional pairwise evaluation.
We evaluated ActiveGED under budget constraints of 30\% and 50\% of the total possible pairwise comparisons, where we set the random budget ratio as 0.5. The results, presented in Table~\ref{table:activeged_results}, demonstrate that ActiveGED achieves competitive performance with significantly fewer evaluations. 
For example, under a 50\% budget, ActiveGED retains most of the performance benefits of full \OURS while cutting the number of pairwise comparisons in half.

\begin{table*}[t]
\caption{Performance comparison of ActiveGED under different budget constraints. ActiveGED achieves competitive performance with significantly fewer pairwise evaluations.}
\centering
\begin{tabular}{lcccccc}
\toprule
\textbf{Evaluators Set of \OURS} & \textbf{HumanEval} & \textbf{AlpacaEval} & \textbf{MATH} & \textbf{GSM8k} & \textbf{GAIA} & \textbf{Avg.} \\
\midrule
Random                & 57.93 & 29.58 & 72.75 & 84.67 & 11.52 & 51.29 \\
ActiveGED (30\%)       & 67.28 & 30.96 & 74.65 & 85.73 & 11.39 & 54.00 \\
ActiveGED (50\%)       & 68.62 & 31.91 & 74.87 & 87.06 & 12.08 & 54.91 \\
\OURS (Full Budget)                   & \textbf{70.73} & \textbf{32.44} & \textbf{75.58} & \textbf{88.18} & \textbf{13.33} & \textbf{56.05} \\
\bottomrule
\end{tabular}
\label{table:activeged_results}
\end{table*}

The algorithm behind ActiveGED is outlined in Algorithm~\ref{alg:activeged}. It begins by randomly selecting a portion of the budget to initialize the preference graph and then iteratively selects the most informative edges based on uncertainty, as estimated using PageRank scores. This approach balances exploration (random sampling) and exploitation (uncertainty-based selection) to efficiently construct an accurate preference graph.

ActiveGED demonstrates that by carefully selecting the most informative pairs, it is possible to achieve competitive performance with significantly fewer evaluations. This makes \OURS more scalable and practical for real-world applications where computational resources are limited.

\section{Mitigating Evaluator Biases in \OURS}
\label{sec:bias_mitigation}

Evaluator biases, such as position bias and token bias, can significantly impact the accuracy and fairness of the preference graphs used in \OURS. Addressing these biases is crucial for ensuring reliable evaluations and robust performance. In this section, we describe the strategies employed in our framework to mitigate these biases and discuss potential areas for future improvement.
Position bias arises when evaluators exhibit a preference for a particular position in a pairwise comparison, such as consistently favoring the first or second option regardless of content. To counter this, we explicitly include both orderings of each question and its candidate answers in the response selection setting. Specifically, for a question \( Q \) with two candidate answers \( A_1 \) and \( A_2 \), we evaluate both configurations:

\begin{itemize}
    \item \textbf{Order 1:}
    \begin{verbatim}
    Question: [Question Text]
    
    Answer 1: [Answer Text A1]
    
    Answer 2: [Answer Text A2]
    
    Which one is better?
    \end{verbatim}
    
    \item \textbf{Order 2:}
    \begin{verbatim}
    Question: [Question Text]
    
    Answer 1: [Answer Text A2]
    
    Answer 2: [Answer Text A1]
    
    Which one is better?
    \end{verbatim}
\end{itemize}

By testing both (Q, A1, A2) and (Q, A2, A1), we ensure that any positional preferences of the evaluators are balanced out when constructing the preference graph. This approach minimizes the impact of position bias on the overall rankings. Token bias occurs when an evaluator exhibits a latent preference for a specific option or label (e.g., consistently favoring "Option A" over "Option B"). Our framework implicitly mitigates token bias through the aforementioned position-swapping strategy, which prevents evaluators from associating a fixed label with a specific position. By averaging the results across both orderings, any systematic preference for a particular label is effectively neutralized.

\section{Rank ensemble method}
\label{sec:rank_ensemble_method}

\paragraph{Weight score~\citep{adler2002review}}:  
    The Weight Score method assigns a score to each vertex \( v \) based on its position in each ranking \( \mathcal{R}_i \). For a vertex \( v \) in ranking \( \mathcal{R}_i \), the score is given by:

\begin{equation}
    S_i(v) = l_i - r_i(v) + 1
\end{equation}

    where \( l_i \) is the length of ranking \( \mathcal{R}_i \) and \( r_i(v) \) is the rank of vertex \( v \) in \( \mathcal{R}_i \). If \( v \) is not present in \( \mathcal{R}_i \), \( S_i(v) = 0 \). The total score for each vertex across all rankings is:

\begin{equation}
    T(v) = \sum_{i=1}^{k} S_i(v)
\end{equation}

    The final consensus ranking \( \mathcal{R}^* \) is obtained by sorting the vertices \( v \) in descending order of \( T(v) \).

 \paragraph{Kemeny and weighted Kemeny~\citep{kemeny1959mathematics}}:  
    The Kemeny method seeks a consensus ranking \( \mathcal{R}^* \) that minimizes the total pairwise disagreements between \( \mathcal{R}^* \) and the input rankings, measured using the Kendall $\tau$-distance:

\begin{equation}
    \mathcal{R}^* = \arg\min_{\mathcal{R}} \sum_{i=1}^{k} \tau(\mathcal{R}, \mathcal{R}_i)
\end{equation}

    The Weighted Kemeny method introduces a weight \( \alpha_i \) for each ranking \( \mathcal{R}_i \), reflecting its importance or reliability:

\begin{equation}
    \mathcal{R}^* = \arg\min_{\mathcal{R}} \sum_{i=1}^{k} \alpha_i \cdot \tau(\mathcal{R}, \mathcal{R}_i)
\end{equation}

    Here, the goal is to minimize the weighted Kendall tau distance, emphasizing rankings with higher weights.

\paragraph{Pairwise majority and weighted pairwise majority}~\citep{caragiannis2016noisy}:  
The Pairwise Majority (PM) method determines a consensus ranking \( \mathcal{R}^* \) by maximizing the number of pairwise agreements with the input rankings. For each pair of vertices \( (v_i, v_j) \), the goal is to ensure that the majority of rankings agree with their relative order in \( \mathcal{R}^* \):

    \begin{equation}
    \scalebox{0.7}{$
    \displaystyle
    \mathcal{R}^* = \arg\max_{\mathcal{R}} \sum_{i < j}
    \left( \sum_{p=1}^{k} \mathds{1}(\mathcal{R}_p(v_i) < \mathcal{R}_p(v_j)) \right)
    \cdot \mathds{1}(\mathcal{R}(v_i) < \mathcal{R}(v_j))
    $}
    \end{equation}

    The Weighted Pairwise Majority method incorporates weights \( \alpha_p \) to account for the reliability of each ranking \( \mathcal{R}_p \):

    \begin{equation}
    \scalebox{0.7}{$
    \displaystyle
    \mathcal{R}^* = \arg\max_{\mathcal{R}} \sum_{i < j}
    \left( \sum_{p=1}^{k} \alpha_p \cdot \mathds{1}(\mathcal{R}_p(v_i) < \mathcal{R}_p(v_j)) \right)
    \cdot \mathds{1}(\mathcal{R}(v_i) < \mathcal{R}(v_j))
    $}
    \end{equation}

    In both methods, the objective is to maximize the (weighted) pairwise agreement between the consensus ranking and the input rankings.

\section{Cost-Performance Analysis of \OURS{}}
\label{sec:cost_detail}
We conducted a detailed cost-performance analysis based on publicly available API pricing from Amazon Bedrock.
API Cost Comparison: For our \OURS{} method, we utilize three relatively lightweight evaluators: Llama3-8B, Mistral-7B, and Qwen2-7B. Since pricing information for Qwen2 is not currently available on Amazon Bedrock, we conservatively adopt the pricing of Llama3 as a proxy. Specifically, the cost per 1 million input tokens is \$0.09, and the cost per 1 million output tokens is \$0.175.
In contrast, using a single strong evaluator such as Qwen2-72B is significantly more expensive, with costs of \$0.23 per 1 million input tokens and \$0.40 per 1 million output tokens. This comparison shows that our \OURS{} method can achieve superior evaluation performance at only approximately 39.13\% of the input cost and 43.75\% of the output cost.
Performance vs. Budget Trade-off: As demonstrated in Appendix~\ref{sec:activeged}, \OURS{} is capable of outperforming strong evaluators like Qwen2-72B even when utilizing only 30–50\% of the evaluation budget. This efficiency arises from effectively aggregating and denoising multiple weaker signals.
The computational overhead involved in graph fusion and denoising steps within \OURS{} is minimal compared to the inference time of large language models, making \OURS{} a cost-effective and practical solution for scalable preference evaluations.

\begin{table*}[t]
\caption{Performance comparison of \OURS with multi-evaluator baselines PRD and ChatEval. \OURS achieves the highest score on each dataset and the best overall average.}
\centering
\begin{tabular}{lcccccc}
\toprule
\textbf{Method} & \textbf{HumanEval} & \textbf{AlpacaEval} & \textbf{MATH} & \textbf{GSM8k} & \textbf{GAIA} & \textbf{Avg.} \\
\midrule
PRD      & 66.53 & 30.33 & 74.89 & 86.75 & 11.83 & 54.07 \\
ChatEval & 66.42 & 29.71 & 74.62 & 86.66 & 11.53 & 53.79 \\
\OURS      & \textbf{70.73} & \textbf{32.44} & \textbf{75.58} & \textbf{88.18} & \textbf{13.33} & \textbf{56.05} \\
\bottomrule
\end{tabular}
\label{table:ged_vs_prd_chateval}
\end{table*}
\section{Comparison with More Baselines}

We extend our experimental analysis to include two recent multi-evaluator baselines: {PRD}~\cite{Li2023PRDPR} and  {ChatEval}~\cite{chan2023chateval}. Both methods are designed to improve evaluation reliability by aggregating feedback from multiple evaluators. PRD focuses on pairwise ranking distillation, while ChatEval employs evaluator voting mechanisms. For a fair comparison, we evaluate all methods using the same set of models: Llama3-8B, Mistral-7B, and Qwen2-7B.
Table~\ref{table:ged_vs_prd_chateval} summarizes the results across five standard benchmark datasets. \OURS consistently outperforms both PRD and ChatEval across all tasks.
These results highlight the effectiveness of \OURS’s structural approach. Unlike debate-style or voting-based methods, \OURS constructs and denoises preference graphs to enforce global consistency, resulting in both improved empirical performance and stronger theoretical guarantees.

\section{\OURS{} \  with Strong Evaluators}

To examine whether \OURS{} can benefit from stronger evaluators, we conducted additional experiments using three high-capacity language models: GPT-4-o-mini, GPT-4-o, and Gemini 2.0. These evaluators were used to construct individual and ensemble preference graphs under the \OURS{} framework.
The results, presented in Figure~\ref{fig:strong_evaluator_effect}, demonstrate clear improvements of \OURS{} over each individual evaluator across multiple tasks. This confirms that \OURS{} not only performs well with lightweight models but also enhances the evaluation quality when applied to strong LLMs such as GPT-4 and Gemini.
These findings support the generality and scalability of our framework. By aggregating and denoising even high-quality but noisy judgments, \OURS{} consistently improves robustness and overall evaluation accuracy.

\begin{figure*}[h]
\centering
\includegraphics[width=1.0\linewidth]{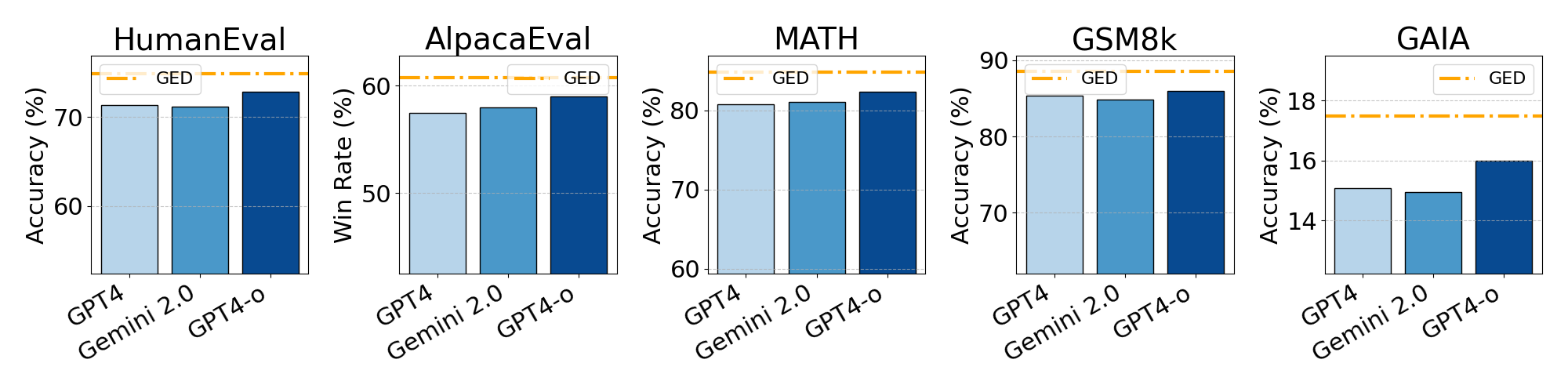}
\caption{Performance of \OURS{} using strong evaluators: GPT-4-o-mini, GPT-4-o, and Gemini 2.0. \OURS{} ensemble improves upon each individual model across tasks.
}
\label{fig:strong_evaluator_effect}
\end{figure*}

\section{ Comparison with Weighted Voting Baselines}

We further evaluate \OURS{} against two intuitive multi-evaluator aggregation baselines: {Weighted Majority Voting} and {Weighted Average of Ratings}. Both baselines utilize the same set of evaluators (Llama3-8B, Mistral-7B, and Qwen2-7B) and apply weight-based schemes to either majority votes or numerical scores.
Weighted Majority Voting assigns each evaluator a weight and determines the preferred candidate based on weighted voting outcomes across all pairwise comparisons. Weighted Average of Ratings aggregates normalized point-wise scores by evaluator weights to rank responses.
As shown in Figure~\ref{fig:weighted_baselines_comparison}, \OURS{} consistently outperforms both baselines across all evaluated tasks. This demonstrates that leveraging the structure of preference graphs and enforcing global consistency through denoising leads to more accurate and robust evaluations compared to simple weighted heuristics.

\begin{figure*}[h]
\centering
\includegraphics[width=1.0\linewidth]{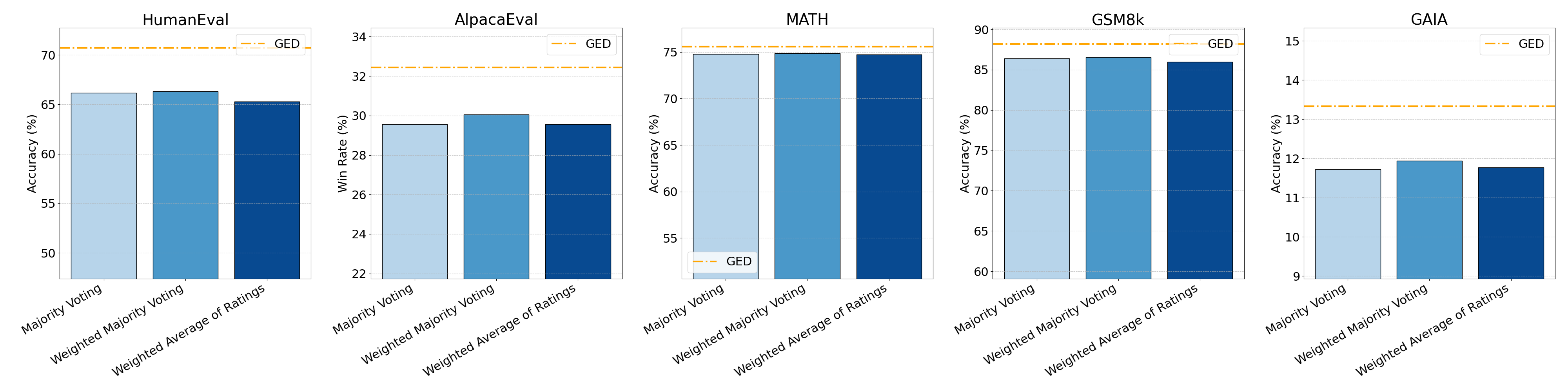}
\caption{Performance comparison of GED with Weighted Majority Voting and Weighted Average of Ratings. GED outperforms both across tasks.}
\label{fig:weighted_baselines_comparison}
\end{figure*}

\section{Prompt Template}
\label{sec:prompt-template}

\paragraph{Prompt for Response Selection.}
In this section, we provide detailed prompt templates used for response selection across five datasets: HumanEval~\citep{chen2021evaluating}, AlpacaEval~\citep{dubois2024alpacafarm}, MATH~\citep{hendrycks2021measuring}, GSM8k~\citep{chen2021evaluating}, and GAIA~\citep{mialon2023gaia}. These prompts include both a system prompt to establish the evaluator's context and a user prompt tailored to the specific task requirements. Each prompt is designed to guide the evaluators in comparing two candidate responses based on task-specific criteria such as correctness, clarity, efficiency, relevance, and completeness. The templates are shown in Table~\ref{table:PromptHumanEval}, Table~\ref{table:PromptAlpacaEval}, Table~\ref{table:PromptMATH}, Table~\ref{table:PromptGSM8k}, and Table~\ref{table:PromptGAIA}.

\paragraph{Prompt for Model Ranking.}
This section presents the prompt template used in the Model Ranking task. The template is designed to evaluate and compare responses generated by different models for a given instruction, based on criteria such as accuracy, clarity, completeness, and helpfulness. The evaluation process involves analyzing two candidate responses and identifying which one better fulfills the requirements of the instruction. The detailed prompt for the Model Ranking  is provided in Table~\ref{table:PromptModelRanking}.

\paragraph{Prompt for Instruct Tuning.}

In this section, we provide the prompt template used for data selection in the Instruct Tuning task. The goal is to select the most appropriate response for each instruction from multiple candidates, ensuring that the selected responses are helpful, harmless, and relevant. We use the HH-RLHF dataset~\citep{bai2022training}, which contains human preference data aimed at training language models to be both helpful and harmless. The detailed prompt used by evaluators to assess and compare candidate responses is presented in Table~\ref{table:PromptInstructTuning}.

\clearpage

\begin{table*}[h]
\renewcommand{\arraystretch}{1.5}
\caption{Prompt template for evaluating programming solutions on the HumanEval dataset.}
\centering
\begin{tabular}{m{0.85\linewidth}} %
\begin{tcolorbox}[mybox, title=Prompt for HumanEval]
{\large\textbf{System Prompt:}}

You are an expert programmer and code reviewer. Your task is to evaluate code solutions for programming problems. Assess each solution based on its correctness, efficiency, readability, and adherence to best coding practices.

\vspace{3mm}

{\large \textbf{User Prompt:}}

Please compare the following two code solutions to the given programming problem. For each solution, evaluate whether it produces correct outputs for all edge cases, whether it is efficient in terms of time and space complexity, and whether the code is clean, well-documented, and follows best practices. Identify any errors or areas for improvement.

\textbf{Programming Problem:}
[Problem Description]

\textbf{Solution A:}
[Candidate Solution A]

\textbf{Solution B:}
[Candidate Solution B]

\textbf{Question:}
Which solution is better and why? Provide a detailed comparison focusing on correctness, efficiency, readability, and coding standards.

\end{tcolorbox}
\end{tabular}
\label{table:PromptHumanEval}
\end{table*}

\begin{table*}[t]
\renewcommand{\arraystretch}{1.5}
\caption{Prompt template for assessing instruction-following responses on the AlpacaEval dataset.}
\centering
\begin{tabular}{m{0.85\linewidth}} %
\begin{tcolorbox}[mybox, title=Prompt for AlpacaEval]
{\large\textbf{System Prompt:}}

You are an AI assistant trained to assess and compare responses to user instructions. Your evaluations should be based on accuracy, clarity, completeness, and helpfulness.

\vspace{3mm}

{\large \textbf{User Prompt:}}

Please compare the following two responses to the given instruction. Analyze each response for how well it follows the instruction, the accuracy of the information provided, the clarity of the explanation, and the overall helpfulness to the user. Point out any errors, omissions, or areas where the response could be improved.

\textbf{Instruction:}
[Instruction Text]

\textbf{Response A:}
[Candidate Response A]

\textbf{Response B:}
[Candidate Response B]

\textbf{Question:}
Which response better addresses the instruction and why? Provide a detailed comparison focusing on the criteria mentioned above.

\end{tcolorbox}
\end{tabular}
\label{table:PromptAlpacaEval}
\end{table*}

\begin{table*}[t]
\renewcommand{\arraystretch}{1.5}
\caption{Prompt template for evaluating mathematical solutions on the MATH dataset.}
\centering
\begin{tabular}{m{0.85\linewidth}} %
\begin{tcolorbox}[mybox, title=Prompt for MATH]
{\large\textbf{System Prompt:}}

You are a mathematician and educator skilled at evaluating mathematical solutions. Assess the correctness, completeness, and clarity of the following solutions to the math problem. Pay attention to the logical reasoning steps, the mathematical accuracy, and the clarity of explanations.

\vspace{3mm}

{\large \textbf{User Prompt:}}

Please evaluate the following two solutions to the given math problem. For each solution, analyze whether the reasoning is correct, if all necessary steps are included, and if the explanations are clear and easy to understand. Identify any errors or misconceptions.

\textbf{Math Problem:}
[Problem Description]

\textbf{Solution A:}
[Candidate Solution A]

\textbf{Solution B:}
[Candidate Solution B]

\textbf{Question:}
Which solution is better and why? Provide a detailed comparison focusing on correctness, completeness, and clarity.

\end{tcolorbox}
\end{tabular}
\label{table:PromptMATH}
\end{table*}

\begin{table*}[t]
\renewcommand{\arraystretch}{1.5}
\caption{Prompt template for assessing multi-step reasoning answers on the GSM8k dataset.}
\centering
\begin{tabular}{m{0.85\linewidth}} %
\begin{tcolorbox}[mybox, title=Prompt for GSM8k]
{\large\textbf{System Prompt:}}

You are a teacher specializing in elementary mathematics. Evaluate student answers to math word problems for correctness and quality of reasoning. Consider whether the student has correctly understood the problem, applied appropriate mathematical operations, and provided clear explanations for each step.

\vspace{3mm}

{\large \textbf{User Prompt:}}

Please compare the following two answers to the given math word problem. For each answer, assess the accuracy of the solution, the appropriateness of the reasoning steps, and the clarity of the explanations. Highlight any mistakes or areas for improvement.

\textbf{Math Word Problem:}
[Problem Description]

\textbf{Answer A:}
[Candidate Answer A]

\textbf{Answer B:}
[Candidate Answer B]

\textbf{Question:}
Which answer is more accurate and better explained, and why? Provide a detailed comparison focusing on the criteria mentioned above.

\end{tcolorbox}
\end{tabular}
\label{table:PromptGSM8k}
\end{table*}

\begin{table*}[t]
\renewcommand{\arraystretch}{1.5}
\caption{Prompt template for evaluating complex question answers on the GAIA dataset.}
\centering
\begin{tabular}{m{0.85\linewidth}} %
\begin{tcolorbox}[mybox, title=Prompt for GAIA]
{\large\textbf{System Prompt:}}

You are an expert in complex problem-solving and knowledge retrieval. Assess the following answers for accuracy, relevance, depth, and comprehensiveness in response to the complex question. Consider whether the answers provide correct information, cover all aspects of the question, and are well-articulated.

\vspace{3mm}

{\large \textbf{User Prompt:}}

Please evaluate the following two answers to the given question. For each answer, analyze the correctness of the information provided, the relevance to the question asked, the depth of the explanation, and the overall quality of the response. Note any inaccuracies, omissions, or areas where the answer could be improved.

\textbf{Question:}
[Complex Question]

\textbf{Answer A:}
[Candidate Answer A]

\textbf{Answer B:}
[Candidate Answer B]

\textbf{Question:}
Which answer provides a better response to the question and why? Provide a detailed comparison focusing on the criteria mentioned above.

\end{tcolorbox}
\end{tabular}
\label{table:PromptGAIA}
\end{table*}


\begin{table*}[t]
\renewcommand{\arraystretch}{1.5}
\caption{Prompt template in the Model Ranking.}
\centering
\begin{tabular}{m{0.85\linewidth}} %
\begin{tcolorbox}[mybox, title=Prompt for Model Ranking]
{\large\textbf{System Prompt:}}

You are an AI assistant trained to assess and compare responses to user instructions. Your evaluations should be based on accuracy, clarity, completeness, and helpfulness.

\vspace{3mm}

{\large \textbf{User Prompt:}}

Please compare the following two responses to the given instruction. Analyze each response for how well it follows the instruction, the accuracy of the information provided, the clarity of the explanation, and the overall helpfulness to the user. Point out any errors, omissions, or areas where the response could be improved.

\textbf{Instruction:}
[Instruction Text]

\textbf{Response A:}
[Candidate Response A]

\textbf{Response B:}
[Candidate Response B]

\textbf{Question:}
Which response better addresses the instruction and why? Provide a detailed comparison focusing on the criteria mentioned above.

\end{tcolorbox}
\end{tabular}
\label{table:PromptModelRanking}
\end{table*}

\begin{table*}[t]
\renewcommand{\arraystretch}{1.5}
\caption{Prompt template for evaluating responses in the Instruct Tuning.}
\centering
\begin{tabular}{m{0.85\linewidth}} %
\begin{tcolorbox}[mybox, title=Prompt for Instruct Tuning]
{\large\textbf{System Prompt:}}

You are a highly skilled AI assistant trained to evaluate and compare responses to user instructions. Your evaluations should focus on helpfulness, harmlessness, and relevance.

\vspace{3mm}

{\large \textbf{User Prompt:}}

Please compare the following two responses to the given instruction. For each response, assess the following aspects:

\textbf{Helpfulness:} Does the response effectively address the instruction and provide useful, accurate information?

\textbf{Harmlessness:} Does the response avoid any harmful, offensive, or inappropriate content?

\textbf{Relevance:} Is the response directly related to the instruction without unnecessary or irrelevant information?

Provide your analysis for each aspect, noting any issues or areas for improvement.

\textbf{Instruction:}
[Instruction Text]

\textbf{Response A:}
[Candidate Response A]

\textbf{Response B:}
[Candidate Response B]

\textbf{Question:}
Which response better satisfies the criteria above and why? Provide a detailed explanation supporting your choice, focusing on helpfulness, harmlessness, and relevance.

\end{tcolorbox}
\end{tabular}
\label{table:PromptInstructTuning}
\end{table*}

\end{document}